\documentclass[11pt]{article}

\usepackage[preprint]{acl}

\usepackage{times}
\usepackage{latexsym}
\usepackage[T1]{fontenc}
\usepackage[utf8]{inputenc}
\usepackage{microtype}
\usepackage{inconsolata}
\usepackage{graphicx}

\usepackage{amsmath}
\usepackage{amsfonts}
\usepackage{amssymb}
\usepackage{algorithm}
\usepackage{algpseudocode}
\usepackage{calc}

\usepackage{caption}
\usepackage[normalem]{ulem}
\usepackage{booktabs}
\usepackage{multirow}
\usepackage{makecell}
\usepackage{enumitem}
\usepackage{amsthm}
\usepackage{multirow}

\theoremstyle{definition}
\newtheorem{proposition}{Proposition}

\newcommand{\attn}[1]{\texttt{attn.#1}}
\newcommand{\mlp}[1]{\texttt{mlp.#1}}

\title{MoARa: Module-Aware Rank Allocation and Structure-Preserving Decomposition for Low-Rank LLM Pre-training}

\author{Keunyoung Kim \\
  Seoul National University \\
  \texttt{keunyoung.kim@snu.ac.kr} \\\And
  Nojun Kwak \\
  Seoul National University \\
  \texttt{nojunk@snu.ac.kr} \\}

\begin{document}
\maketitle

% ============================================================================
% ============================================================================
% ⬇️⬇️⬇️⬇️ ABSTRACT ⬇️⬇️⬇️⬇️
% ============================================================================
% ============================================================================

\begin{abstract}
Low-rank gradient projection reduces the optimizer-state memory cost of large language model (LLM) pretraining, but the steps and wall-clock time needed to reach a target quality remain a meaningful axis for improvement. 
We attribute this to two design choices in existing methods: the projection-rank budget is allocated uniformly across Transformer modules with heterogeneous projection sensitivity, and projecting a raw gradient attenuates its magnitude and direction jointly. 
We propose \textbf{MoARa}, which combines a static profiling-based \emph{module-aware projection-rank allocation} with a \emph{block-wise magnitude--direction decomposition}; the default block size is set in the neighborhood of the attention head dimension. 
Across five Transformer architectures spanning Llama, Qwen, and DeepSeek at 300M to 7B scales, GaLore with MoARa reaches standard GaLore's final perplexity in 37\% fewer steps and 34\% less wall-clock time on Llama 2 7B, with only 0.2\% peak reserved memory overhead under standard graph compilation. 
Across the six low-rank pretraining methods we evaluate, module-aware rank allocation alone delivers directionally consistent step reductions on all six.
On compatible hosts, the two-component design reaches up to 41.7\% step reduction and 37.1\% wall-clock reduction.
\end{abstract}

% ============================================================================
% ============================================================================
% ⬆️⬆️⬆️⬆️ ABSTRACT ⬆️⬆️⬆️⬆️
% ============================================================================
% ============================================================================

% ============================================================================
% ============================================================================
% ⬇️⬇️⬇️⬇️ FIGURE: 7B_MAIN ⬇️⬇️⬇️⬇️
% ============================================================================
% ============================================================================

\begin{figure}
    \centering
    \includegraphics[width=1\linewidth]{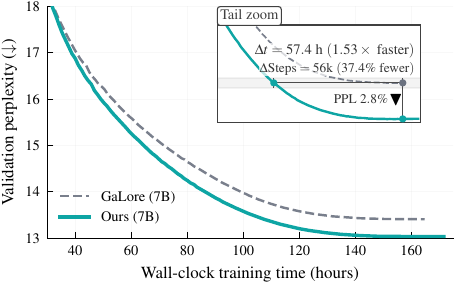}
    \caption{\textbf{Accelerated pre-training convergence at the 7B scale.} Validation PPL versus wall-clock time for GaLore with MoARa and standard GaLore on Llama~2 7B; the inset (\textit{Tail zoom}) marks the convergence-time and final-perplexity gains highlighted in the figure.} 
    \label{fig:7b_main}
\end{figure}

% ============================================================================
% ============================================================================
% ⬆️⬆️⬆️⬆️ FIGURE: 7B_MAIN ⬆️⬆️⬆️⬆️
% ============================================================================
% ============================================================================

% ============================================================================
% ============================================================================
% ⬇️⬇️⬇️⬇️ Sec 1. INTRODUCTION ⬇️⬇️⬇️⬇️
% ============================================================================
% ============================================================================

\section{Introduction}
\label{sec:introduction}

Recent advances in large language models (LLMs) have been driven by scaling model size, data, and training compute, with memory overhead emerging as a major systems bottleneck. Under Adam~\citep{kingma2015adam}-style full-parameter training, optimizer states alone can consume memory comparable to or exceeding model weights, since maintaining 1st- and 2nd-moment statistics requires two additional tensors of the same size as each parameter. At the billion-parameter scale, this overhead substantially raises hardware requirements and limits practical scalability.

A growing line of work addresses this bottleneck by projecting Adam's optimizer states into a low-dimensional subspace, retaining full-parameter learning while substantially reducing optimizer-state memory. Representative methods include GaLore~\citep{zhao2024galore}, which periodically recomputes the projection via SVD, and several subsequent extensions that refine how this subspace is maintained~\citep{liang2024osd, robert2025ldadam, rajabi2025subtrackpp, zhang2025qgalore, chen2025fira}. These methods now approach the validation perplexity of full-rank training at substantially lower memory cost, but the number of optimization steps and the wall-clock time required to reach a given quality target remains a meaningful axis along which low-rank pretraining can be made more practical.

We identify two design limitations that constrain the convergence speed of existing methods. \emph{First}, the projection-rank budget is allocated uniformly across Transformer modules, implicitly treating the seven attention and MLP projections ($\texttt{attn.q}$/$\texttt{k}$/$\texttt{v}$/$\texttt{o}$ and $\texttt{mlp.up}$/$\texttt{gate}$/$\texttt{down}$) as equally sensitive to projection. As we show, the alignment between each module's full gradient and its low-rank reconstruction is sharply heterogeneous, so a uniform allocation over-provisions some modules while leaving others as bottlenecks. \emph{Second}, projecting a gradient into a low-rank subspace attenuates both its local magnitude and direction. As a result, the per-block scale information that helps the optimizer adapt is lost together with the directional information.

To address these two limitations, we propose \textbf{Module-Aware Rank Allocation and Structure-Preserving Decomposition for Low-Rank LLM Pretraining (MoARa)}, a memory-efficient pretraining framework with two complementary components. 
The first, \emph{module-aware projection-rank allocation}, redistributes a fixed total projection-rank budget across Transformer module types. 
Based on a profiling alignment diagnostic, we consolidate this redistribution into a static allocation rule (\S\ref{sec:rank_allocation}). 
The second, \emph{block-wise magnitude--direction decomposition}, partitions each gradient row into contiguous blocks and separates per-block scale from direction before projection. 
The block-local scale is retained as a separate optimizer state outside the low-rank projection, which is the specific structure preserved by our method (\S\ref{sec:block_decomposition}). 
We set the default block size in the neighborhood of the attention head dimension, which our analysis identifies as a balanced regime and our empirical sweep supports as a robust plateau (\S\ref{sec:theoretical_foundations}).

We evaluate MoARa across five Transformer architectures (Llama 2~\citep{touvron2023llama2openfoundation}, Llama 3.2~\citep{grattafiori2024llama3herdmodels}, Qwen2.5~\citep{qwen2025qwen25technicalreport}, Qwen3~\citep{yang2025qwen3technicalreport}, DeepSeek-V2~\citep{deepseekai2024deepseekv2strongeconomicalefficient}; 300M to 7B), at sequence lengths 256 and 2048, and across six low-rank pretraining methods. On Llama 2 7B (Figure~\ref{fig:7b_main}), MoARa reaches standard GaLore's final perplexity 57.4 hours earlier --- 37\% fewer steps and 34\% less wall-clock time --- with only 0.2\% additional peak reserved memory under standard graph compilation. Our main contributions are:

% ============================================================================
% ============================================================================
% ⬇️⬇️⬇️⬇️ Sec 1. INTRODUCTION - BULLET POINT ⬇️⬇️⬇️⬇️
% ============================================================================
% ============================================================================

\begin{itemize}[itemsep=2pt, topsep=2pt, leftmargin=*]

\item \textbf{A module-aware low-rank pretraining framework.} 
We introduce MoARa, which couples module-aware projection-rank allocation with structure-preserving block-wise magnitude--direction decomposition. 
Both components are motivated by testable claims about the information geometry of low-rank gradient projection.
The attention-head dimension provides an architecture-informed reference for the default block size.

\item \textbf{Accelerated convergence across scale, architectures, and sequence length.} On Llama 2, MoARa applied to GaLore yields up to $1.64\times$ step speedup from 350M to 7B; at 7B this corresponds to 37\% fewer steps and approximately 57.4 hours saved. Comparable reductions extend to sequence length 2048 and to four other recent Transformer architectures, with $+0.2\%$ peak reserved memory at 7B under standard graph compilation (Table~\ref{tab:memory_breakdown}).

\item \textbf{Component-specific cross-method transferability.} 
Module-aware rank allocation transfers as a directionally consistent improvement across all six low-rank pretraining methods we evaluate (GaLore~\citep{zhao2024galore}, Fira~\citep{chen2025fira}, Q-GaLore~\citep{zhang2025qgalore}, SubTrack++~\citep{rajabi2025subtrackpp}, LDAdam~\citep{robert2025ldadam}, OSD~\citep{liang2024osd}), with 8.3\%--28.3\% step reductions on the five non-GaLore methods.
In contrast, structure-preserving decomposition is host-dependent and shows limited benefit under Q-GaLore.
When combined with projection rank allocation on compatible hosts, the framework reaches up to 41.7\% step and 37.1\% wall-clock reduction.
\end{itemize}

% ============================================================================
% ============================================================================
% ⬆️⬆️⬆️⬆️ Sec 1. INTRODUCTION ⬆️⬆️⬆️⬆️
% ============================================================================
% ============================================================================

% ============================================================================
% ============================================================================
% ⬇️⬇️⬇️⬇️ Sec 2. RELATED WORK ⬇️⬇️⬇️⬇️
% ============================================================================
% ============================================================================

\section{Related Work}
\label{sec:related_work}

\paragraph{Memory-efficient pretraining via low-rank gradient projection.}
Reducing optimizer-state memory under full-parameter updates is a central challenge in LLM pretraining. GaLore~\citep{zhao2024galore} introduces gradient low-rank projection, periodically recomputing the projection via SVD. Subsequent methods primarily refine how the subspace is maintained: OSD~\citep{liang2024osd} replaces periodic SVD with online PCA and provides the first convergence guarantee for arbitrary projection update rules; LDAdam~\citep{robert2025ldadam} introduces a projection-aware update rule for optimizer states across changing subspaces with a generalized error-feedback mechanism; SubTrack++~\citep{rajabi2025subtrackpp} tracks the gradient subspace on the Grassmannian manifold with projection-aware moments; Q-GaLore~\citep{zhang2025qgalore} combines INT4-quantized projection matrices with layer-wise adaptive SVD frequency; and Fira~\citep{chen2025fira} restores full-rank gradient updates while keeping low-rank optimizer state through norm-based scaling. Recent work further refines subspace selection through importance sampling and moment orthogonalization~\citep{zhang2025breakingfrozensubspaceimportance,refael2025sumo}. Across this family, however, the projection-rank budget is distributed \emph{uniformly} across Transformer modules, implicitly assuming that the seven attention and MLP projections share equal optimization sensitivity. Our work explicitly challenges this assumption and reallocates the budget module-wise via a baseline profiling diagnostic.

\paragraph{Memory efficiency along orthogonal axes.}
Memory-efficient pretraining has also been pursued along axes orthogonal to optimizer-side SVD projection. CoLA \citep{liu2025cola} restructures attention and MLP projections as low-rank autoencoders, which alters the model architecture itself and places it outside our fixed-architecture, fixed-budget comparison frame. APOLLO \citep{zhu2025apollo} uses random-projection-based structured AdamW~\citep{loshchilov2019adamw} scaling whose auxiliary state carries no per-module $r$-dimensional subspace, so it lies in a different design space from the methods compared in \S\ref{sec:experiments}.

\paragraph{Gradient and weight decomposition.}
Decoupling magnitude from direction has a long history of stabilizing high-dimensional optimization. Weight Normalization~\citep{salimans2016weightnorm} reparameterizes weights as $w = g \cdot v/\|v\|$ for improved gradient conditioning, and DoRA~\citep{liu2024dora} extends this decoupling to parameter-efficient fine-tuning; both operate on weights rather than gradients. On the optimizer side, Adafactor~\citep{shazeer2018adafactor} factorizes second-moment statistics row/column-wise, and VLoRP~\citep{wang2025vlorp} varies the granularity of low-rank gradient projection. MoARa is, to our knowledge, the first method to apply magnitude--direction decomposition to gradients in low-rank pretraining, where the goal is preserving directional information through the projection bottleneck.

\paragraph{Module-aware rank allocation in Transformers.}
A growing literature establishes that Transformer modules are structurally heterogeneous and benefit from differentiated low-rank treatment. In post-hoc compression, LoRAP~\citep{li2024lorap} and $\mathrm{A}^3$~\citep{wong2025a3} report distinct low-rank characteristics between attention sub-layers and feed-forward components; in parameter-efficient fine-tuning, AdaLoRA~\citep{zhang2023adalora}, ALoRA~\citep{liu2024alora}, and ARA~\citep{xv2025ara} adaptively allocate the rank budget across weight matrices based on importance scores. 
All of these methods target compression or fine-tuning of already-trained models, where importance can be measured on stable representations.
From-scratch pretraining is fundamentally different, as rank allocation must account for Adam's evolving moment states.

To our knowledge, MoARa is the first to bring budget-preserving, module-aware rank allocation to from-scratch LLM pretraining. This is done via a static profiling phase to maintain optimizer history without destabilizing training. The same module-wise sensitivity pattern persists across the five Transformer architectures we evaluate, suggesting the asymmetry is structural rather than incidental. A summary of how MoARa positions relative to these methods is provided in Appendix~\ref{appx:positioning} (Table~\ref{tab:positioning}).

% ============================================================================
% ============================================================================
% ⬆️⬆️⬆️⬆️ Sec 2. RELATED WORK ⬆️⬆️⬆️⬆️
% ============================================================================
% ============================================================================

% ============================================================================
% ============================================================================
% ⬇️⬇️⬇️⬇️ Sec 3. METHODOLOGY ⬇️⬇️⬇️⬇️
% ============================================================================
% ============================================================================

\section{Methodology}
\label{sec:methodology}
 
We present \textbf{MoARa}, a memory-efficient pretraining framework with two components: Block-wise Magnitude--Direction Decomposition and Projection Rank Reallocation via Subspace Alignment. Algorithm~\ref{alg:moara_full} summarizes the overall procedure.

% ============================================================================
% ============================================================================
% ⬇️⬇️⬇️⬇️ ALGORITHM 1 ⬇️⬇️⬇️⬇️
% ============================================================================
% ============================================================================

\begin{algorithm}[t]
    \caption{Overall procedure of MoARa}
    \label{alg:moara_full}
    \small
    \newcommand{\StatexIndent}[1]{%
      \Statex \hspace{\algorithmicindent}%
      \parbox[t]{\dimexpr\linewidth-\algorithmicindent\relax}{#1}%
    }
    {\fontsize{9.5pt}{10.8pt}\selectfont
    \begin{algorithmic}[1]
    \Require Model $\mathcal{M}$ with module types $\mathcal{T}$; base rank $r_{\mathrm{base}}$; 
    total budget $R_{\mathrm{total}}$; profiling step set $\mathcal{S}_p = \{x, 2x, \ldots, N_p\}$; 
    donor set $\mathcal{D}$, receiver set $\mathcal{R}$, transfer amount $\Delta r$; 
    SVD interval $T_{\mathrm{SVD}}$; block size $B$; 
    learning rate $\eta$.
    \Ensure Trained weights $W$.
    \Statex \textbf{Setup phase (executed once before training).}
    \StatexIndent{\textbf{Phase 1. Profiling.}}
    \StatexIndent{Run uniform-rank GaLore for $N_p$ steps; at each $t \in \mathcal{S}_p$ and each weight $W$ of type $\tau$, log $G_t(W)$ and $G_{\mathrm{recon},t}(W)$.}
    \Statex
    \StatexIndent{\textbf{Phase 2. Diagnostic verification.}}
    \StatexIndent{Compute $S_\tau$ for all $\tau\in\mathcal{T}$; validate that $\mathcal{R}$ contains
    low-$S_\tau$ modules and $\mathcal{D}$ matches donor ablations.}
    \Statex
    \StatexIndent{\textbf{Phase 3. Budget-preserving reallocation.}}
    \StatexIndent{Set $r_\tau \gets r_{\mathrm{base}} - \Delta r$ for $\tau \in \mathcal{D}$; 
    $r_\tau \gets r_{\mathrm{base}} + \Delta r \cdot |\mathcal{D}|/|\mathcal{R}|$ for $\tau \in \mathcal{R}$; 
    $r_\tau \gets r_{\mathrm{base}}$ otherwise.}
    \Statex
    \Statex \textbf{Training phase (with fixed ranks from Phase 3).}
    \For{each training step $s$}
        \For{each target weight $W$ of type $\tau$}
            \State $G \gets \nabla_W \mathcal{L}$
            \State $M, V \gets \mathrm{Decompose}(G, B)$
            \If{$s \mod T_{\mathrm{SVD}} = 0$}
                \State $P_\tau \gets \mathrm{TruncatedSVD}(V, r_\tau)$
            \EndIf
            \State $\tilde V \gets P_\tau^\top V$
            \State $\Delta \tilde V \gets \mathrm{Optimizer}_V(\tilde V)$
            \State $\Delta V \gets P_\tau \, \Delta \tilde V$
            \State $\Delta M \gets \mathrm{Optimizer}_M(M)$
            \State $W \gets W - \eta \bigl(\mathrm{Rep}_B(\Delta M) \odot \Delta V\bigr)$
        \EndFor
    \EndFor
    \State \Return $W$
    \end{algorithmic}
    }
\end{algorithm}

% ============================================================================
% ============================================================================
% ⬆️⬆️⬆️⬆️ ALGORITHM 1 ⬆️⬆️⬆️⬆️
% ============================================================================
% ============================================================================

\subsection{Block-wise Magnitude and Direction Decomposition}
\label{sec:block_decomposition}
 
To mitigate the memory cost of optimizer states, GaLore~\citep{zhao2024galore} projects each gradient $G \in \mathbb{R}^{m \times n}$ into a low-rank subspace as $G_{\mathrm{proj}} = P^\top G$ via a projection matrix $P \in \mathbb{R}^{m \times r}$ ($r \ll m, n$), and maintains optimizer states only for the projected gradient. Projecting raw gradients into this fixed subspace, however, forces magnitude and direction to be compressed jointly, attenuating both factors together and potentially suppressing low-energy but still informative updates. Decomposing the gradient before projection reduces this coupled attenuation, so that the projected branch focuses on directional preservation while a separate state tracks magnitude. A naive row-wise decomposition is overly coarse for Transformer weight matrices: applying a single scalar magnitude to an entire row imposes an artificial dependency among heterogeneous features and can distort the projected subspace. Element-wise normalization, at the other extreme, is unnecessarily expensive. Thus block size is introduced to control the granularity of the decomposition. We adopt an intermediate granularity and decompose the gradient into block-wise magnitude and direction.
 
Let $G \in \mathbb{R}^{m \times n}$.
When $B \nmid n$, we right-zero-pad each row to the smallest column dimension
$n_{\mathrm{pad}} \ge n$ such that $B \mid n_{\mathrm{pad}}$.
Let $k=n_{\mathrm{pad}}/B$ and denote the padded gradient by
$G_{\mathrm{pad}} \in \mathbb{R}^{m \times n_{\mathrm{pad}}}$.
We partition each padded row into $k$ contiguous $1\times B$ blocks:
\begin{equation}
\begin{aligned}
g_{i,j}
&=
(G_{\mathrm{pad}})_{i,\,jB:(j+1)B}
\in \mathbb{R}^{B},\\
&\quad \text{where} \quad j=0,\ldots,k-1.
\end{aligned}
\end{equation}
For each block, we compute a scalar magnitude
$m_{i,j}=\lVert g_{i,j}\rVert_2$
and a normalized direction block
\[
(V_{\mathrm{pad}})_{i,\,jB:(j+1)B}
=
\frac{(G_{\mathrm{pad}})_{i,\,jB:(j+1)B}}
{m_{i,j}+\epsilon}
\]
for a small $\epsilon>0$.
Collecting all magnitudes yields
$M\in\mathbb{R}^{m\times k}$, and
$V_{\mathrm{pad}}\in\mathbb{R}^{m\times n_{\mathrm{pad}}}$.
With $\operatorname{Rep}_B(M)$ denoting the matrix obtained by repeating
each $m_{i,j}$ over its $B$ columns, the padded representation satisfies
\[
G_{\mathrm{pad}}
\approx
\operatorname{Rep}_B(M)\odot V_{\mathrm{pad}},
\]
with exact equality recovered as $\epsilon\rightarrow0$.
The padded columns are discarded after decomposition and recomposition,
so the resulting direction and update tensors retain the original shape.
 
\paragraph{Choice of $B$: head-dimension-informed default.}
We adopt $B$ in the neighborhood of the attention head dimension $d_{\mathrm{head}}$ as an architecture-informed scale reference, rather than requiring exact head alignment. We provide a theoretical analysis of the underlying trade-off in \S\ref{sec:theoretical_foundations}, empirical validation in Appendix~\ref{appx:block_size_sensitivity}, and specific values per model scale in Appendix~\ref{appx:block_size_config}.
 
\subsection{Projection Rank Reallocation via Subspace Alignment}
\label{sec:rank_allocation}
 
\paragraph{Architectural intuition and alignment diagnostic.}
Our second design is motivated by a mismatch between uniform projection-rank allocation and the heterogeneous sensitivity of Transformer modules. If projection distortion is structurally non-uniform, a uniform rank allocation over-provisions some modules while leaving others as bottlenecks. We form an a priori expectation about which modules tolerate projection-rank reduction by reading the decoder block: $\attn{q}$ and $\attn{k}$ parameterize token-to-token compatibility scores (routing), $\attn{v}$ carries representational content and $\attn{o}$ integrates the multi-head outputs, and the MLP projections transform and propagate representational content, with $\mlp{down}$ acting as the bottleneck that compresses the expanded MLP representation back into the model dimension. We therefore expected the routing pair $\attn{q}$/$\attn{k}$ to be more tolerant to rank reduction than the content-carrying modules, and $\mlp{down}$ to be most sensitive. To test this, we measure how well the projected subspace preserves each module's gradient direction via cosine similarity. With $G_{\mathrm{recon}} = \mathrm{Rep}_B(M) \odot (P P^\top V)$, the alignment at step $t$ is
\begin{equation}
S(W, t) = \cos\bigl(G_t(W),\, G_{\mathrm{recon},t}(W)\bigr),
\label{eq:module_alignment}
\end{equation}
and the module-wise alignment score $S_\tau = \mathbb{E}_{t,\, W \in \tau}[S(W, t)]$ averages over profiling steps and matrix instances of module type $\tau$.
 
\paragraph{Donor and receiver selection from joint evidence.}
The observed $S_\tau$ ordering places $\mlp{down}$ lowest by a wide margin, the routing pair $\attn{q}$/$\attn{k}$ at moderate values, and the content/integration modules $\attn{v}$/$\attn{o}$ highest, with $\mlp{up}$ and $\mlp{gate}$ in the middle (Table~\ref{tab:cosine_stride_ranking}).  The budget-receiver role of $\mlp{down}$ is strongly supported, but the cosine ranking alone would suggest $\attn{v}$/$\attn{o}$ as budget-donors --- in conflict with our architectural intuition. A single-module rank-reduction ablation (Appendix~\ref{appx:donor_sensitivity}) resolves the conflict: reducing the rank of $\attn{q}$ or $\attn{k}$ causes only a small increase in final perplexity, while reducing $\attn{v}$, $\attn{o}$, $\mlp{up}$, or $\mlp{gate}$ causes a substantially larger one. The high cosine alignment of $\attn{v}$/$\attn{o}$ therefore reflects a structural redundancy the projection preserves well, rather than actual robustness to capacity loss. We accordingly select the donor set $\mathcal{D} = \{\attn{q}, \attn{k}\}$ and the receiver set $\mathcal{R} = \{\mlp{down}\}$.
 
\paragraph{Static budget-preserving reallocation.}
Given a total projection-rank budget $R_{\mathrm{total}}$, we assign $r_\tau \gets r_{\mathrm{base}} - \Delta r$ for $\tau \in \mathcal{D}$, $r_\tau \gets r_{\mathrm{base}} + \Delta r \cdot |\mathcal{D}|/|\mathcal{R}|$ for $\tau \in \mathcal{R}$, and the uniform baseline rank otherwise; this preserves $\sum_\tau r_\tau = R_{\mathrm{total}}$ by construction. We adopt $\Delta r = r_{\mathrm{base}}/2$ as a balanced default that transfers a nontrivial fraction of donor capacity while preserving half of each donor's baseline rank, without claiming this value is optimal. The resulting integer ranks are fixed throughout training. This static allocation is a deliberate design choice. In Adam-style low-rank training, changing ranks online changes the dimensionality of the optimizer states --- newly added directions require fresh first- and second-moment initialization, while removed directions discard accumulated optimizer history, which can destabilize the training trajectory. Static allocation captures persistent module-wise sensitivity while maintaining optimizer-state continuity. 
Profiling runs once before target training and is excluded from the target-training wall-clock trajectories; a 1k-step run ($\sim$11~min on 350M in our settings) suffices to recover the full assignment (Appendix~\ref{appx:profiling_cost}).
 
\subsection{Theoretical Foundations}
\label{sec:theoretical_foundations}

Two information-geometric properties, both built on the cosine perspective of $S_\tau$, motivate MoARa's design choices.

\emph{(i) Carrying magnitude outside the projection isolates block-wise cosine distortion to direction.}
Let $\widetilde{V}=P_VP_V^\top V$ and let $\widetilde{v}_{i,b}$ denote the corresponding block of $\widetilde{V}$.
For every block for which both vectors are nonzero,
\[
\cos(G_{i,b},G_{\mathrm{recon},i,b})
=
\cos(v_{i,b},\widetilde{v}_{i,b}).
\]
Thus, the scalar magnitude factor cancels from the block-wise cosine alignment, and any projection-induced alignment change is expressed through the direction branch.

\emph{(ii) Block size induces a statistical--coupling trade-off.} 
Small $B$ leaves block magnitudes dominated by per-coordinate noise. 
At the extreme $B=1$, the direction signal becomes sign-valued and loses all coordinate-wise magnitude variation. 
Large $B$ aggregates heterogeneous coordinates. 
In this regime, the unit-norm direction $v_{i,b} = G_{i,b}/m_{i,b}$ attenuates low-magnitude coordinates by the dominant ones within the same block. 
Setting $B$ in the neighborhood of $d_{\mathrm{head}}$ provides an architecture-informed operating point in the balanced regime between the two extremes.

The two observations above provide an information-geometric rationale for MoARa's design, rather than a formal convergence guarantee for MoARa-augmented Adam.
Full derivations are given in Appendix~\ref{appx:theory}.

% ============================================================================
% ============================================================================
% ⬆️⬆️⬆️⬆️ Sec 3. METHODOLOGY ⬆️⬆️⬆️⬆️
% ============================================================================
% ============================================================================

% ============================================================================
% ============================================================================
% ⬇️⬇️⬇️⬇️ FIGURE: SCALE_TRAJECTORIES ⬇️⬇️⬇️⬇️
% ============================================================================
% ============================================================================

\begin{figure}[t]
    \centering
    \includegraphics[width=1\linewidth]{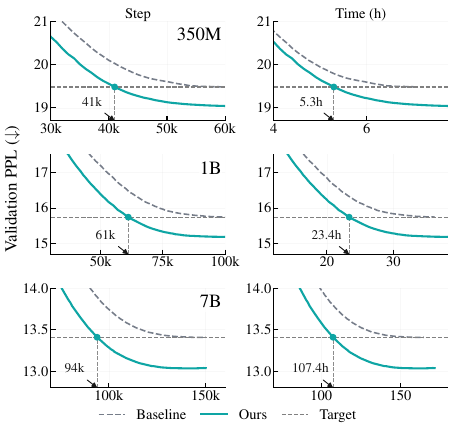}
    \caption{\textbf{Pre-training validation perplexity across model scales.} 
    The left and right columns plot perplexity versus training steps and wall-clock time, respectively. 
    Annotated points mark where our method reaches the baseline's final perplexity.}
    \label{fig:scale_trajectories}
\end{figure}

% ============================================================================
% ============================================================================
% ⬆️⬆️⬆️⬆️ FIGURE: SCALE_TRAJECTORIES ⬆️⬆️⬆️⬆️
% ============================================================================
% ============================================================================

% ============================================================================
% ============================================================================
% ⬇️⬇️⬇️⬇️ Sec 4. EXPERIMENTS ⬇️⬇️⬇️⬇️
% ============================================================================
% ============================================================================ 
 
\section{Experiments}
\label{sec:experiments}
 
\subsection{Experimental Setup}
\label{sec:experimental_setup}
 
\textbf{Models and Architectures.}
For the main scaling study, we adopt the Llama 2~\citep{touvron2023llama2openfoundation} architecture at 350M, 1B, and 7B scales. To evaluate architectural generalization, we further consider Llama 3.2 (300M)~\citep{grattafiori2024llama3herdmodels}, Qwen2.5 (350M)~\citep{qwen2025qwen25technicalreport}, Qwen3 (350M)~\citep{yang2025qwen3technicalreport}, and the dense variant of DeepSeek-V2 (350M)~\citep{deepseekai2024deepseekv2strongeconomicalefficient}, scaled to be comparable to Llama 2 350M.
 
\paragraph{Baselines.}
Our primary comparisons are within the low-rank pretraining family, where GaLore serves as our main host. To assess transferability beyond GaLore, we additionally evaluate MoARa applied to five other low-rank methods in our cross-method ablation: Fira~\citep{chen2025fira}, Q-GaLore~\citep{zhang2025qgalore}, SubTrack++~\citep{rajabi2025subtrackpp}, LDAdam~\citep{robert2025ldadam}, and OSD~\citep{liang2024osd}.

\paragraph{Projection rank allocation protocol.}
For the Llama~2 scaling study, we profile the 350M configuration and reuse the resulting donor/receiver assignment at 1B and 7B, while the numerical ranks follow each scale's $r_{base}$. 
For Llama 3.2, Qwen2.5, and Qwen3, we fix Q and K as donors and Down as the receiver. 
For our custom dense DeepSeek-V2 variant, we use Q(b) and KV(b) as donors and Down as the receiver. 
These mappings are fixed before target training without target-architecture profiling or performance-based rank tuning.
 
\paragraph{Training details.}
All models are pretrained from scratch on the English C4~\citep{raffel2023exploringlimitstransferlearning} dataset. To ensure fair comparison, we control all training hyperparameters; detailed configurations are provided in Appendix~\ref{app_sec:exp_setup_detail}. Unless otherwise stated, all methods use the same total projection rank budget.

\subsection{MoARa on GaLore across Scale, Architecture, and Sequence Length}
\label{sec:main_result}
 
We evaluate MoARa applied to GaLore along three axes: model scale, architecture family, and pretraining sequence length.
 
\paragraph{Pretraining efficiency at scale.}
Figure~\ref{fig:scale_trajectories} reports validation perplexity trajectories with respect to both training steps and wall-clock time at 350M, 1B, and 7B scales on Llama 2. 
Across all three scales, GaLore with MoARa converges faster than standard GaLore in both axes. 
On Llama 2 7B (Figure~\ref{fig:7b_main}), GaLore requires 150k steps to reach its final perplexity while GaLore with MoARa reaches the same target at 94k steps, reducing the required steps by 37\% and the wall-clock time by 34\%. 
These wall-clock values report target-training time only. 
Including the one-time 1k-step 350M profiling cost ($\sim$11 min), the non-amortized time-to-target is 5.47 h at 350M versus 7.13 h for GaLore, while the wall-clock reduction at 7B remains approximately 34.7\%.
When the same assignment is reused across $K$ target runs, the amortized time-to-target is $T_{train}$ + $T_{profile}$ / $K$ (Appendix~\ref{appx:profiling_cost}).
As a secondary fixed-budget result, when both methods are trained for 150k steps, GaLore with MoARa continues to improve after the standard GaLore curve plateaus and achieves a further 2.8\% gain in final perplexity.
This acceleration is consistent across other scales and architectures, yielding up to $1.64\times$ step speedup from 350M to 7B in Llama 2 (Figure~\ref{fig:speedup_across_scale} in Appendix~\ref{app_sec:additional_experimental_validation}). 
At 7B, the reduction in steps-to-target from 150k to 94k produces the net wall-clock saving despite a 4.19\% per-step overhead.

% ============================================================================
% ============================================================================
% ⬇️⬇️⬇️⬇️ TABLE: Llama 2 1B benchmark results ⬇️⬇️⬇️⬇️
% ============================================================================
% ============================================================================ 

\begin{table}[t]
    \centering
    \caption{\textbf{Llama 2 1B benchmark results.}
    We report means over three seeds. $\Delta$ denotes relative change. ``Win'' counts the number of seeds for which our method outperforms GaLore. Full results are in Appendix~\ref{appx:downstream}.}
    \label{tab:downstream_1b}
    \small
    \renewcommand{\arraystretch}{1.12}
    \setlength{\tabcolsep}{0pt}
    \begin{tabular*}{\columnwidth}{@{\extracolsep{\fill}}lrrrc@{}}
    \toprule
    \textbf{Task} & \textbf{Base} & \textbf{Ours} & \textbf{$\Delta$} & \textbf{Win} \\
    \midrule
    \multicolumn{5}{@{}l}{\textit{Language modeling} $\downarrow$} \\
    C4 ppl      & 86.88 & 82.62 & $-4.90\%$ & 3/3 \\
    WikiText2 ppl     & 64.71 & 60.27 & $-6.86\%$ & 2/3 \\
    \addlinespace[2pt]
    \midrule
    \multicolumn{5}{@{}l}{\textit{Commonsense / reasoning} $\uparrow$} \\
    HellaSwag   & .386 & .398 & $+3.11\%$ & 3/3 \\
    WinoGrande  & .513 & .531 & $+3.51\%$ & 3/3 \\
    PIQA        & .677 & .682 & $+0.74\%$ & 2/3 \\
    BoolQ       & .546 & .558 & $+2.20\%$ & 2/3 \\
    ARC-E       & .425 & .433 & $+1.88\%$ & 2/3 \\
    ARC-C       & .252 & .248 & $-1.59\%$ & 1/3 \\
    \bottomrule
    \end{tabular*}
\end{table}

% ============================================================================
% ============================================================================
% ⬆️⬆️⬆️⬆️ TABLE: Llama 2 1B benchmark results ⬆️⬆️⬆️⬆️
% ============================================================================
% ============================================================================ 
 
\paragraph{Generalization across recent architectures.}
To verify that MoARa is not specific to one model family, we evaluate it on several recent Transformer architectures at the 300--350M scale, all of which depart from the standard multi-head attention used in Llama~2. 
Llama 3.2 and the Qwen family adopt Grouped-Query Attention, while DeepSeek-V2 uses Multi-head Latent Attention; both alter the dimensional ratios and parameter distributions of the \texttt{attn} matrices. 
Despite these differences, the Llama-2-derived allocation rule transfers through the predefined structural mapping described in \S\ref{sec:experimental_setup}, without target-architecture profiling or performance-based rank tuning, and achieves substantial reductions in steps-to-target across all tested architectures, including approximately 43\% on Llama 3.2, 33\% on the Qwen models, and 28\% on DeepSeek-V2 (Appendix~\ref{appx:architecture_results}), with proportional wall-clock reductions.

% ============================================================================
% ============================================================================
% ⬇️⬇️⬇️⬇️ TABLE 2: MEMORY BREAKDOWN ⬇️⬇️⬇️⬇️
% ============================================================================
% ============================================================================
 
\begin{table}[t]
    \centering
    \caption{\textbf{Component-wise memory cost (eager mode, MiB).} Treatment columns show relative change versus baseline. Under standard graph compilation at 7B, the max allocated/reserved overheads reduce to negligible levels. Full breakdown in Appendix~\ref{appx:memory}.}
    \label{tab:memory_breakdown}
    \small
    \setlength{\tabcolsep}{2.2pt}
    \renewcommand{\arraystretch}{1.08}
    \begin{tabular*}{\columnwidth}{@{\extracolsep{\fill}}clrrrr@{}}
    \toprule
    \textbf{Scale} & \textbf{Metric} & \textbf{Base} & \textbf{+Decomp} & \textbf{+Rank} & \textbf{Full} \\
    \midrule
    \multirow{3}{*}{350M}
    & Max alloc.  & 47258 & {\footnotesize +0.05\%} & {\footnotesize +0.00\%} & {\footnotesize +0.12\%} \\
    & Max reserv. & 48154 & {\footnotesize +0.16\%} & {\footnotesize -0.05\%} & {\footnotesize +0.12\%} \\
    & Opt. state  &   539 & {\footnotesize +13.4\%} & {\footnotesize +7.5\%}  & {\footnotesize +20.9\%} \\
    \midrule
    \multirow{3}{*}{1B}
    & Max alloc.  & 51611 & {\footnotesize +0.00\%} & {\footnotesize +0.00\%} & {\footnotesize +0.00\%} \\
    & Max reserv. & 54988 & {\footnotesize +0.29\%} & {\footnotesize +0.26\%} & {\footnotesize +0.57\%} \\
    & Opt. state  &  1652 & {\footnotesize +8.7\%}  & {\footnotesize +9.7\%}  & {\footnotesize +18.4\%} \\
    \midrule
    \multirow{3}{*}{7B}
    & Max alloc.  & 60899 & {\footnotesize +6.39\%} & {\footnotesize +6.39\%} & {\footnotesize +6.39\%} \\
    & Max reserv. & 62122 & {\footnotesize +7.43\%} & {\footnotesize +7.22\%} & {\footnotesize +7.43\%} \\
    & Opt. state  &  7177 & {\footnotesize +5.4\%}  & {\footnotesize +12.0\%} & {\footnotesize +17.4\%} \\
    \bottomrule
    \end{tabular*} 
\end{table}

% ============================================================================
% ============================================================================
% ⬆️⬆️⬆️⬆️ TABLE 2: MEMORY BREAKDOWN ⬆️⬆️⬆️⬆️
% ============================================================================
% ============================================================================
 
\paragraph{Robustness at sequence length 2048.}
The above scaling and architecture results follow the GaLore-standard sequence length of 256.
To test whether the efficiency gains persist beyond this short-context setting, we re-evaluate GaLore with MoARa against standard GaLore at sequence length 2048 on Llama~2 350M.
GaLore with MoARa reaches standard GaLore's final validation perplexity in 35.0\% fewer steps and 33.3\% less wall-clock time, and improves final perplexity by 2.65\% (Appendix~\ref{appx:context_2048}).

% ============================================================================
% ============================================================================
% ⬇️⬇️⬇️⬇️ TABLE 3: CROSS-METHOD ABLATION ⬇️⬇️⬇️⬇️
% ============================================================================
% ============================================================================

\begin{table*}[t]
    \centering
    \caption{\textbf{Cross-method ablation at Llama 2 350M, 60k steps.} 
    Each cell reports final PPL and step-to-/wall-clock-to-target. 
    Target is each host's vanilla final perplexity; ``---'' indicates that the target was not reached within the 60k-step budget.}
    \label{tab:cross_method}
    \small
    \setlength{\tabcolsep}{2pt}
    \renewcommand{\arraystretch}{1.12}
    \begin{tabular*}{\textwidth}{@{\extracolsep{\fill}}lcccccc@{}}
    \toprule
    \textbf{Condition} & \textbf{GaLore} & \textbf{Fira} & \textbf{Q-GaLore} & \textbf{SubTrack++} & \textbf{OSD} & \textbf{LDAdam} \\
    \midrule
    Vanilla
    & \makecell{19.48 \\ 60k ($-$) \\ 7.13h ($-$)}
    & \makecell{16.85 \\ 60k ($-$) \\ 7.10h ($-$)}
    & \makecell{20.64 \\ 60k ($-$) \\ 8.27h ($-$)}
    & \makecell{16.19 \\ 60k ($-$) \\ 6.28h ($-$)}
    & \makecell{21.66 \\ 60k ($-$) \\ 8.78h ($-$)}
    & \makecell{17.60 \\ 60k ($-$) \\ 9.33h ($-$)} \\
    \midrule
    $+$ Rank
    & \makecell{19.27 \\ 47k ($\downarrow$22.0\%) \\ 5.58h ($\downarrow$21.7\%)}
    & \makecell{16.79 \\ 55k ($\downarrow$8.3\%) \\ 6.60h ($\downarrow$7.0\%)}
    & \makecell{\textbf{20.45} \\ \textbf{45k ($\downarrow$25.0\%)} \\ \textbf{6.19h ($\downarrow$25.2\%)}}
    & \makecell{\textbf{15.57} \\ \textbf{43k ($\downarrow$28.3\%)} \\ \textbf{4.50h ($\downarrow$28.4\%)}}
    & \makecell{21.27 \\ 44k ($\downarrow$26.7\%) \\ 6.97h ($\downarrow$20.6\%)}
    & \makecell{\textbf{17.19} \\ \textbf{44k ($\downarrow$26.7\%)} \\ \textbf{6.99h ($\downarrow$25.1\%)}} \\
    \midrule
    $+$ Decom
    & \makecell{19.45 \\ 54k ($\downarrow$9.6\%) \\ 7.06h ($\downarrow$1.0\%)}
    & \makecell{16.75 \\ 53k ($\downarrow$11.7\%) \\ 6.79h ($\downarrow$4.4\%)}
    & \makecell{20.81 \\ --- \\ ---}
    & \makecell{28.29 \\ --- \\ ---}
    & \makecell{20.76 \\ 37k ($\downarrow$38.3\%) \\ 5.41h ($\downarrow$38.4\%)}
    & \makecell{24.01 \\ --- \\ ---} \\
    \midrule
    $+$ Full (MoARa)
    & \makecell{\textbf{19.04} \\ \textbf{41k ($\downarrow$31.7\%)} \\ \textbf{5.29h ($\downarrow$25.8\%)}}
    & \makecell{\textbf{16.70} \\ \textbf{51k ($\downarrow$15.0\%)} \\ \textbf{6.65h ($\downarrow$6.3\%)}}
    & \makecell{20.46 \\ 45k ($\downarrow$25.0\%) \\ 6.26h ($\downarrow$24.3\%)}
    & \makecell{19.17 \\ --- \\ ---}
    & \makecell{\textbf{20.42} \\ \textbf{35k ($\downarrow$41.7\%)} \\ \textbf{5.52h ($\downarrow$37.1\%)}}
    & \makecell{20.24 \\ --- \\ ---} \\
    \bottomrule
    \end{tabular*}
\end{table*}

% ============================================================================
% ============================================================================
% ⬆️⬆️⬆️⬆️ TABLE 3: CROSS-METHOD ABLATION ⬆️⬆️⬆️⬆️
% ============================================================================
% ============================================================================

\paragraph{Stability of the default configuration.}
The default configuration is supported by complementary diagnostics and sensitivity checks.
The cosine-alignment analysis consistently identifies \texttt{mlp.down} as the receiver, while the single-module rank-reduction ablation identifies Q/K as the most robust donors.
The 1k-step profile recovers the same donor/receiver assignment as the longer profiles, and this assignment remains unchanged across profiling strides of 50--400 steps.
For decomposition, $B\in\{32,64,128\}$ forms a near-optimal plateau with little variation across the three settings.
We use $\Delta r=r_{\mathrm{base}}/2$ as a conservative budget-preserving default, while $T_{\mathrm{SVD}}=200$ follows standard GaLore and is not tuned for MoARa.
 
\paragraph{Robustness at the 1B scale and zero-shot downstream evaluation.}
As a limited sanity check, we assess whether the faster pretraining trajectory causes systematic downstream degradation.
We train Llama 2 1B with three random seeds for both standard GaLore and GaLore with MoARa and evaluate the 100k-step checkpoints on a suite of zero-shot downstream benchmarks~\citep{zellers2019hellaswag, sakaguchi2020winogrande, bisk2020piqa, clark2019boolq, clark2018think, lin2022truthfulqa}.
Table~\ref{tab:downstream_1b} reports the results over three seeds. MoARa matches or improves upon standard GaLore on the majority of tasks across seeds on several primary metrics. 
The full benchmark set is reported in Appendix~\ref{appx:downstream}.

\paragraph{Memory overhead in practice.}
These gains are achieved without materially compromising the memory advantage of low-rank pretraining. 
Table~\ref{tab:memory_breakdown} summarizes the component-wise memory cost in eager mode, with the full eager- and compiled-mode breakdown reported in Appendix~~\ref{appx:memory}.
All wall-clock experiments use \texttt{torch.compile}.
Under this compiled setting at the 7B scale, MoARa increases peak allocated memory by only $+0.04$~MiB and peak reserved memory by $0.2\%$ relative to GaLore. 
Persistent optimizer-state memory is a separate, compilation-independent metric.
At 7B, it increases by 1,250~MiB ($+17.4\%$), with the contributions from block-wise decomposition and module-aware rank allocation combining approximately additively. 
Peak allocated and reserved memory are dominated by model parameters and activations, so the eager-mode peak increase at 7B largely disappears under standard graph compilation.

\subsection{Component-wise Cross-Method Transferability}
\label{sec:cross_method}
 
To examine whether MoARa's two components are tied to GaLore's specific projection mechanism, we apply each component independently to five additional low-rank pretraining methods spanning distinct mechanism classes: Fira~\citep{chen2025fira}, Q-GaLore~\citep{zhang2025qgalore}, SubTrack++~\citep{rajabi2025subtrackpp}, LDAdam~\citep{robert2025ldadam}, and OSD~\citep{liang2024osd}. 
Combined with GaLore, this yields a 6 $\times$ 4 ablation at the 350M scale on C4. 
We report step reduction and wall-clock reduction to reach the host's vanilla final perplexity (Table~\ref{tab:cross_method}).
In summary, rank allocation improves all six hosts, whereas the added benefit of decomposition is host-dependent and is limited under Q-GaLore.
 
\paragraph{Module-aware rank allocation transfers as a directionally consistent improvement.}
Across all six host methods, rank allocation alone reaches the host's vanilla final PPL with strictly fewer steps and less wall-clock time. On the five non-GaLore hosts, step reduction ranges from 8.3\% (Fira) to 28.3\% (SubTrack++), and wall-clock reduction ranges from 7.0\% to 28.4\%; the two reductions remain closely matched in each case, since rank allocation only redistributes the projection budget rather than enlarging it. The directional consistency across six independent host mechanisms suggests that the module-wise differential in projection rank sensitivity is shared across the low-rank pretraining family rather than specific to GaLore's update rule.
 
\paragraph{The added benefit of block-wise decomposition is host-dependent.}
Three patterns emerge from Table~\ref{tab:cross_method}: on hosts with slowly or discretely evolving projections (GaLore, Fira, OSD), the full configuration improves over vanilla on all three hosts and gives the best final PPL and steps-to-target.
Its wall-clock gain additionally reflects the per-step cost of decomposition. 
On OSD, the full configuration achieves 41.7\% step and 37.1\% wall-clock reduction.
On Q-GaLore, decomposition alone produces a small increase in final PPL and the full configuration recovers the rank-allocation-only result. 
On SubTrack++ and LDAdam, whose projection bases change continuously during training, decomposition does not combine stably with the evolving subspace despite multiple remedies, and the full configuration does not reach the host's vanilla target within the 60k-step budget. 
We analyze the host-specific mechanisms underlying all three patterns in \S\ref{sec:analysis}.

% ============================================================================
% ============================================================================
% ⬆️⬆️⬆️⬆️ Sec 4. EXPERIMENTS ⬆️⬆️⬆️⬆️
% ============================================================================
% ============================================================================ 

% ============================================================================
% ============================================================================
% ⬇️⬇️⬇️⬇️ Sec 5. ANALYSIS AND DISCUSSION ⬇️⬇️⬇️⬇️
% ============================================================================
% ============================================================================

\section{Analysis and Discussion}
\label{sec:analysis}

The cross-method ablation of \S\ref{sec:cross_method} revealed a sharp asymmetry between MoARa's two components: rank allocation transfers consistently across all six hosts, while decomposition transfers with strong host-dependence. We discuss the mechanisms responsible for this pattern. As an additional sanity check on the allocation itself, MoARa's sensitivity-guided allocation also outperforms 20 random budget-preserving allocations (Appendix~\ref{appx:randomized_allocation}).

\paragraph{Projection rank allocation operates on an architectural property.}
The module-wise alignment scores $S_\tau$ --- measured during a separate baseline profiling run as the cosine similarity between the full gradient and its low-rank reconstruction --- characterize how much meaningful signal each module's gradient retains inside its projection subspace, independent of how the optimizer subsequently chooses to update along that subspace. 
The consistent benefit across the six evaluated hosts therefore suggests that this module-wise asymmetry is not specific to GaLore's update rule.

\paragraph{Decomposition operates on a host-mechanism property, splitting hosts into three groups.}
The behavior observed in \S\ref{sec:cross_method} separates hosts by how their projection bases evolve over training and by the precision constraints of their numerical representations.

\emph{Group A --- slowly or discretely evolving projections (GaLore~\citep{zhao2024galore}, Fira~\citep{chen2025fira}, OSD~\citep{liang2024osd}).} 
Fira maintains an $r$-dimensional optimizer state and applies a full-rank gradient update via norm-based scaling. This scaling partially restores the column-norm information attenuated by projection. In contrast, our decomposition addresses this information loss through a different path by separating the magnitude pre-projection. 
The Fira result in Table~\ref{tab:cross_method} suggests that its norm-based scaling and our decomposition are compatible in the evaluated setting.
OSD updates its projection basis online via PCA on accumulated statistics.
A plausible interpretation is that its basis evolves gradually enough for decomposition to remain effective in our evaluation.

\emph{Group B --- continuously evolving projections (SubTrack++~\citep{rajabi2025subtrackpp}, LDAdam~\citep{robert2025ldadam}).} Both methods update the projection basis continuously during training. SubTrack++ operates on the Grassmannian manifold, while LDAdam applies a per-step subspace correction. 
Under this continuously evolving projection regime, a plausible explanation is that the separately optimized magnitude state becomes harder to coordinate with direction updates induced by the changing projection basis.
Under the remedies we tested, the full configuration does not reach the host's vanilla final perplexity within the 60k-step budget.

\emph{Group C --- quantized host (Q-GaLore~\citep{zhang2025qgalore}).} 
Q-GaLore combines INT4 projection matrices and INT8 weights with an 8-bit Adam optimizer.
In our precision ablation, increasing only the magnitude-branch optimizer state from 8-bit to FP32 does not recover the benefit of decomposition (Appendix~\ref{appx:qgalore_precision}).
This suggests that optimizer-state precision alone does not explain the degradation. 
The other quantized components of Q-GaLore are therefore plausible compatibility constraints for decomposition.

\paragraph{Two-axis summary.}
The cross-method pattern admits a clean two-axis reading. 
First, rank allocation is architecture-driven.
Across the six evaluated hosts, it consistently improves time-to-target, suggesting that the module-wise asymmetry is not specific to GaLore. 
Second, decomposition is \emph{host-mechanism-driven}. 
Across the evaluated hosts, decomposition is most effective when the projection basis evolves slowly or discretely and when the numerical representation is sufficiently precise.
We treat these as observed compatibility factors rather than necessary or sufficient conditions.
These conditions clarify when each component is likely to transfer within the broader design space of low-rank pretraining methods.

\paragraph{Practical trade-offs.}
Relative to the corresponding vanilla low-rank baselines, MoARa adds a one-time profiling stage and a static module-wise rank configuration.
During target training, projection ranks remain fixed, so no online rank search or optimizer-state resizing is required.
The remaining cost comes from per-step decomposition, recomposition, and the separate magnitude branch.
In the 7B GaLore comparison, these additions result in a 4.19\% per-step overhead.
Our method also adds persistent optimizer state, while its compiled peak-memory overhead remains small.

% ============================================================================
% ============================================================================
% ⬆️⬆️⬆️⬆️ Sec 5. ANALYSIS AND DISCUSSION ⬆️⬆️⬆️⬆️
% ============================================================================
% ============================================================================

% ============================================================================
% ============================================================================
% ⬇️⬇️⬇️⬇️ Sec 6. CONCLUSION ⬇️⬇️⬇️⬇️
% ============================================================================
% ============================================================================
 
\section{Conclusion}
\label{sec:conclusion}
 
We presented \textbf{MoARa}, a low-rank LLM pretraining framework. It combines profiling-based rank allocation with block-wise magnitude--direction decomposition. 
Together, these components improve the use of a fixed projection-rank budget and retain per-block magnitude information.
Evaluated across five Transformer architectures from 300M to 7B, MoARa reaches the GaLore baseline's final perplexity in 37\% fewer steps and 34\% less wall-clock time on Llama 2 7B. 
Across all six evaluated hosts, module-aware rank allocation consistently reduces steps-to-target.
On compatible hosts, the full design of MoARa achieves up to 41.7\% step and 37.1\% wall-clock reduction.
Our analysis indicates that projection rank allocation captures a persistent architecture-related pattern within the evaluated settings, while decomposition shows host-dependent compatibility associated with projection-basis evolution and numerical precision.

% ============================================================================
% ============================================================================
% ⬆️⬆️⬆️⬆️ Sec 6. CONCLUSION ⬆️⬆️⬆️⬆️
% ============================================================================
% ============================================================================

\clearpage

% ============================================================================
% ============================================================================
% ⬇️⬇️⬇️⬇️ LIMITATION ⬇️⬇️⬇️⬇️
% ============================================================================
% ============================================================================

\section*{Limitations}
\label{sec:limitations}

While MoARa effectively mitigates key bottlenecks in low-rank LLM pretraining, we discuss boundaries of our work and identify directions for future research.

\paragraph{Absence of a stable dynamic rank-allocation mechanism.}
MoARa allocates ranks statically through a single profiling phase prior to training, and the question of whether a stable \emph{dynamic} rank-allocation mechanism is achievable remains open. A dynamic schedule that, for instance, expands the capacity of receiver modules after warmup is conceptually appealing, but our empirical investigations show that expanding a module's projection rank mid-training requires initializing the first- and second-moment statistics of Adam for newly spanned orthogonal directions, and we observe that this derails the training trajectory. We hypothesize that model parameters rapidly co-adapt to the initial low-rank bottleneck, and abruptly altering this representational capacity perturbs the accumulated moment states. Designing a stable dynamic mechanism that resolves this momentum--subspace mismatch is left to future work; the static allocation studied here should therefore be understood as the simplest stable instance of the broader module-aware allocation principle rather than its only realization.

\paragraph{Cross-method decomposition compatibility for continuously-evolving projections.}
Our cross-method ablation (\S\ref{sec:cross_method}) shows that block-wise magnitude--direction decomposition does not stably combine with hosts that update their projection basis continuously during training (SubTrack++, LDAdam) under the variants we tested. 
Future work could focus on two key questions. 
First, what conditions allow a magnitude state to be co-transported consistently with continuously-evolving projections? 
Second, does this co-transport rule lead to a stable optimization trajectory?

\paragraph{Formal convergence rate guarantees.}
Our theoretical analysis (\S\ref{sec:theoretical_foundations}, Appendix~\ref{appx:theory}) characterizes the information-geometric properties of MoARa's design but does not establish a formal convergence rate guarantee for MoARa-augmented Adam. Such a result would require tracking the moment dynamics of the magnitude and direction branches jointly and is a natural direction for follow-up theoretical work.

\paragraph{Scope of optimizer evaluation.}
Our experiments focus on low-rank optimizer-state projection methods built on AdamW~\citep{loshchilov2019adamw} -style optimization, which provide a controlled setting for isolating the effects of projection-rank allocation and structure-preserving decomposition.
Evaluating specialized optimizers such as 8-bit AdamW or Muon is important, but would introduce optimizer-specific numerical behavior and tuning choices that could confound the main algorithmic comparison.
We therefore do not claim that MoARa outperforms all memory-efficient optimizers.
A controlled cross-family comparison would require matched memory budgets and separate tuning protocols, which we leave to future work.

\paragraph{Scope of generalization.}
Our experiments focus on decoder-only Transformers, and whether the same allocation transfers to encoder--decoder architectures remains unverified. 
We also do not evaluate transfer across different pre-training datasets. 
Extending the 2048-token study across all model scales or to substantially longer contexts was computationally prohibitive within our available budget, so our long-context evidence is limited to the Llama 2 350M setting.

% ============================================================================
% ============================================================================
% ⬆️⬆️⬆️⬆️ LIMITATION ⬆️⬆️⬆️⬆️
% ============================================================================
% ============================================================================

% ============================================================================
% ============================================================================
% ⬇️⬇️⬇️⬇️ REFERENCE & ACKNOWLEDGMENTS ⬇️⬇️⬇️⬇️
% ============================================================================
% ============================================================================

% camera-ready 시에만 사용
\section*{Acknowledgments}
This work was supported by IITP grants funded by the Government of the Republic of Korea (Ministry of Science and ICT) under Grant Nos. RS-2021-II211343 and RS-2025-25442338, and by the "Advanced GPU Utilization Support Program".

\bibstyle{acl_natbib}
\bibliography{references}

@inproceedings{kingma2015adam,
    title = {{Adam}: A Method for Stochastic Optimization},
    author = {Diederik P. Kingma and Jimmy Ba},
    booktitle = {International Conference on Learning Representations},
    year = {2015},
    url = {https://arxiv.org/abs/1412.6980}
}

@inproceedings{zhao2024galore,
    title = {{GaLore}: Memory-Efficient {LLM} Training by Gradient Low-Rank Projection},
    author = {Jiawei Zhao and Zhenyu Zhang and Beidi Chen and Zhangyang Wang and Anima Anandkumar and Yuandong Tian},
    booktitle = {Proceedings of the 41st International Conference on Machine Learning},
    volume = {235},
    pages = {61121--61143},
    series = {Proceedings of Machine Learning Research},
    publisher = {PMLR},
    year = {2024},
    url = {https://proceedings.mlr.press/v235/zhao24s.html}
}

@inproceedings{liang2024osd,
    title = {Memory-Efficient {LLM} Training with Online Subspace Descent},
    author = {Kaizhao Liang and Bo Liu and Lizhang Chen and Qiang Liu},
    booktitle = {Advances in Neural Information Processing Systems},
    volume = {37},
    year = {2024},
    doi = {10.52202/079017-2054},
    url = {https://proceedings.neurips.cc/paper_files/paper/2024/hash/760d09bcc06b949f5ac4b6a918739aa8-Abstract-Conference.html}
}

@inproceedings{robert2025ldadam,
    title = {{LDAdam}: Adaptive Optimization from Low-Dimensional Gradient Statistics},
    author = {Thomas Robert and Mher Safaryan and Ionut-Vlad Modoranu and Dan Alistarh},
    booktitle = {International Conference on Learning Representations},
    year = {2025},
    url = {https://proceedings.iclr.cc/paper_files/paper/2025/hash/0e6eb6f6001748b66348b9b53ad434db-Abstract-Conference.html}
}

@inproceedings{rajabi2025subtrackpp,
    title = {{SubTrack++}: Gradient Subspace Tracking for Scalable {LLM} Training},
    author = {Sahar Rajabi and Nayeema Nonta and Sirisha Rambhatla},
    booktitle = {Advances in Neural Information Processing Systems},
    volume = {38},
    year = {2025},
    doi = {10.52202/085713-1060},
    url = {https://proceedings.neurips.cc/paper_files/paper/2025/hash/2d62cb71e87ae340e3ab0e874befcbc2-Abstract-Conference.html}
}

@inproceedings{zhang2025qgalore,
    title = {{Q-GaLore}: Quantized {GaLore} with {INT4} Projection and Layer-Adaptive Low-Rank Gradients},
    author = {Zhenyu Zhang and Ajay Kumar Jaiswal and Lu Yin and Shiwei Liu and Jiawei Zhao and Yuandong Tian and Zhangyang Wang},
    booktitle = {Conference on Parsimony and Learning},
    volume = {280},
    pages = {1035--1050},
    series = {Proceedings of Machine Learning Research},
    publisher = {PMLR},
    year = {2025},
    url = {https://proceedings.mlr.press/v280/zhang25a.html}
}

@inproceedings{chen2025fira,
    title = {{Fira}: Can We Achieve Full-rank Training of {LLM}s Under Low-rank Constraint?},
    author = {Xi Chen and Kaituo Feng and Changsheng Li and Xunhao Lai and Xiangyu Yue and Ye Yuan and Guoren Wang},
    booktitle = {Advances in Neural Information Processing Systems},
    volume = {38},
    year = {2025},
    doi = {10.52202/085713-4026},
    url = {https://proceedings.neurips.cc/paper_files/paper/2025/hash/aeae2c860cbe283ef73344c4ecd52567-Abstract-Conference.html}
}

@misc{touvron2023llama2openfoundation,
    title = {{Llama 2}: Open Foundation and Fine-Tuned Chat Models},
    author = {Hugo Touvron and Louis Martin and Kevin Stone and Peter Albert and Amjad Almahairi and Yasmine Babaei and Nikolay Bashlykov and Soumya Batra and Prajjwal Bhargava and Shruti Bhosale and Dan Bikel and Lukas Blecher and Cristian Canton Ferrer and Moya Chen and Guillem Cucurull and David Esiobu and Jude Fernandes and Jeremy Fu and Wenyin Fu and Brian Fuller and Cynthia Gao and Vedanuj Goswami and Naman Goyal and Anthony Hartshorn and Saghar Hosseini and Rui Hou and Hakan Inan and Marcin Kardas and Viktor Kerkez and Madian Khabsa and Isabel Kloumann and Artem Korenev and Punit Singh Koura and Marie-Anne Lachaux and Thibaut Lavril and Jenya Lee and Diana Liskovich and Yinghai Lu and Yuning Mao and Xavier Martinet and Todor Mihaylov and Pushkar Mishra and Igor Molybog and Yixin Nie and Andrew Poulton and Jeremy Reizenstein and Rashi Rungta and Kalyan Saladi and Alan Schelten and Ruan Silva and Eric Michael Smith and Ranjan Subramanian and Xiaoqing Ellen Tan and Binh Tang and Ross Taylor and Adina Williams and Jian Xiang Kuan and Puxin Xu and Zheng Yan and Iliyan Zarov and Yuchen Zhang and Angela Fan and Melanie Kambadur and Sharan Narang and Aurelien Rodriguez and Robert Stojnic and Sergey Edunov and Thomas Scialom},
    year = {2023},
    eprint = {2307.09288},
    archivePrefix = {arXiv},
    primaryClass = {cs.CL},
    url = {https://arxiv.org/abs/2307.09288}
}

@misc{grattafiori2024llama3herdmodels,
    title = {The {Llama 3} Herd of Models},
    author = {Aaron Grattafiori and Abhimanyu Dubey and Abhinav Jauhri and Abhinav Pandey and Abhishek Kadian and Ahmad Al-Dahle and Aiesha Letman and Akhil Mathur and Alan Schelten and Alex Vaughan and Amy Yang and Angela Fan and Anirudh Goyal and Anthony Hartshorn and Aobo Yang and Archi Mitra and Archie Sravankumar and Artem Korenev and Arthur Hinsvark and Arun Rao and Aston Zhang and Aurelien Rodriguez and Austen Gregerson and Ava Spataru and Baptiste Roziere and Bethany Biron and Binh Tang and Bobbie Chern and Charlotte Caucheteux and Chaya Nayak and Chloe Bi and Chris Marra and Chris McConnell and Christian Keller and Christophe Touret and Chunyang Wu and Corinne Wong and Cristian Canton Ferrer and Cyrus Nikolaidis and Damien Allonsius and Daniel Song and Danielle Pintz and Danny Livshits and Danny Wyatt and David Esiobu and Dhruv Choudhary and Dhruv Mahajan and Diego Garcia-Olano and Diego Perino and Dieuwke Hupkes and Egor Lakomkin and Ehab AlBadawy and Elina Lobanova and Emily Dinan and Eric Michael Smith and Filip Radenovic and Francisco Guzmán and Frank Zhang and Gabriel Synnaeve and Gabrielle Lee and Georgia Lewis Anderson and Govind Thattai and Graeme Nail and Gregoire Mialon and Guan Pang and Guillem Cucurell and Hailey Nguyen and Hannah Korevaar and Hu Xu and Hugo Touvron and Iliyan Zarov and Imanol Arrieta Ibarra and Isabel Kloumann and Ishan Misra and Ivan Evtimov and Jack Zhang and Jade Copet and Jaewon Lee and Jan Geffert and Jana Vranes and Jason Park and Jay Mahadeokar and Jeet Shah and Jelmer van der Linde and Jennifer Billock and Jenny Hong and Jenya Lee and Jeremy Fu and Jianfeng Chi and Jianyu Huang and Jiawen Liu and Jie Wang and Jiecao Yu and Joanna Bitton and Joe Spisak and Jongsoo Park and Joseph Rocca and Joshua Johnstun and Joshua Saxe and Junteng Jia and Kalyan Vasuden Alwala and Karthik Prasad and Kartikeya Upasani and Kate Plawiak and Ke Li and Kenneth Heafield and Kevin Stone and Khalid El-Arini and Krithika Iyer and Kshitiz Malik and Kuenley Chiu and Kunal Bhalla and Kushal Lakhotia and Lauren Rantala-Yeary and Laurens van der Maaten and Lawrence Chen and Liang Tan and Liz Jenkins and Louis Martin and Lovish Madaan and Lubo Malo and Lukas Blecher and Lukas Landzaat and Luke de Oliveira and Madeline Muzzi and Mahesh Pasupuleti and Mannat Singh and Manohar Paluri and Marcin Kardas and Maria Tsimpoukelli and Mathew Oldham and Mathieu Rita and Maya Pavlova and Melanie Kambadur and Mike Lewis and Min Si and Mitesh Kumar Singh and Mona Hassan and Naman Goyal and Narjes Torabi and Nikolay Bashlykov and Nikolay Bogoychev and Niladri Chatterji and Ning Zhang and Olivier Duchenne and Onur Çelebi and Patrick Alrassy and Pengchuan Zhang and Pengwei Li and Petar Vasic and Peter Weng and Prajjwal Bhargava and Pratik Dubal and Praveen Krishnan and Punit Singh Koura and Puxin Xu and Qing He and Qingxiao Dong and Ragavan Srinivasan and Raj Ganapathy and Ramon Calderer and Ricardo Silveira Cabral and Robert Stojnic and Roberta Raileanu and Rohan Maheswari and Rohit Girdhar and Rohit Patel and Romain Sauvestre and Ronnie Polidoro and Roshan Sumbaly and Ross Taylor and Ruan Silva and Rui Hou and Rui Wang and Saghar Hosseini and Sahana Chennabasappa and Sanjay Singh and Sean Bell and Seohyun Sonia Kim and Sergey Edunov and Shaoliang Nie and Sharan Narang and Sharath Raparthy and Sheng Shen and Shengye Wan and Shruti Bhosale and Shun Zhang and Simon Vandenhende and Soumya Batra and Spencer Whitman and Sten Sootla and Stephane Collot and Suchin Gururangan and Sydney Borodinsky and Tamar Herman and Tara Fowler and Tarek Sheasha and Thomas Georgiou and Thomas Scialom and Tobias Speckbacher and Todor Mihaylov and Tong Xiao and Ujjwal Karn and Vedanuj Goswami and Vibhor Gupta and Vignesh Ramanathan and Viktor Kerkez and Vincent Gonguet and Virginie Do and Vish Vogeti and Vítor Albiero and Vladan Petrovic and Weiwei Chu and Wenhan Xiong and Wenyin Fu and Whitney Meers and Xavier Martinet and Xiaodong Wang and Xiaofang Wang and Xiaoqing Ellen Tan and Xide Xia and Xinfeng Xie and Xuchao Jia and Xuewei Wang and Yaelle Goldschlag and Yashesh Gaur and Yasmine Babaei and Yi Wen and Yiwen Song and Yuchen Zhang and Yue Li and Yuning Mao and Zacharie Delpierre Coudert and Zheng Yan and Zhengxing Chen and Zoe Papakipos and Aaditya Singh and Aayushi Srivastava and Abha Jain and Adam Kelsey and Adam Shajnfeld and Adithya Gangidi and Adolfo Victoria and Ahuva Goldstand and Ajay Menon and Ajay Sharma and Alex Boesenberg and Alexei Baevski and Allie Feinstein and Amanda Kallet and Amit Sangani and Amos Teo and Anam Yunus and Andrei Lupu and Andres Alvarado and Andrew Caples and Andrew Gu and Andrew Ho and Andrew Poulton and Andrew Ryan and Ankit Ramchandani and Annie Dong and Annie Franco and Anuj Goyal and Aparajita Saraf and Arkabandhu Chowdhury and Ashley Gabriel and Ashwin Bharambe and Assaf Eisenman and Azadeh Yazdan and Beau James and Ben Maurer and Benjamin Leonhardi and Bernie Huang and Beth Loyd and Beto De Paola and Bhargavi Paranjape and Bing Liu and Bo Wu and Boyu Ni and Braden Hancock and Bram Wasti and Brandon Spence and Brani Stojkovic and Brian Gamido and Britt Montalvo and Carl Parker and Carly Burton and Catalina Mejia and Ce Liu and Changhan Wang and Changkyu Kim and Chao Zhou and Chester Hu and Ching-Hsiang Chu and Chris Cai and Chris Tindal and Christoph Feichtenhofer and Cynthia Gao and Damon Civin and Dana Beaty and Daniel Kreymer and Daniel Li and David Adkins and David Xu and Davide Testuggine and Delia David and Devi Parikh and Diana Liskovich and Didem Foss and Dingkang Wang and Duc Le and Dustin Holland and Edward Dowling and Eissa Jamil and Elaine Montgomery and Eleonora Presani and Emily Hahn and Emily Wood and Eric-Tuan Le and Erik Brinkman and Esteban Arcaute and Evan Dunbar and Evan Smothers and Fei Sun and Felix Kreuk and Feng Tian and Filippos Kokkinos and Firat Ozgenel and Francesco Caggioni and Frank Kanayet and Frank Seide and Gabriela Medina Florez and Gabriella Schwarz and Gada Badeer and Georgia Swee and Gil Halpern and Grant Herman and Grigory Sizov and Guangyi and Zhang and Guna Lakshminarayanan and Hakan Inan and Hamid Shojanazeri and Han Zou and Hannah Wang and Hanwen Zha and Haroun Habeeb and Harrison Rudolph and Helen Suk and Henry Aspegren and Hunter Goldman and Hongyuan Zhan and Ibrahim Damlaj and Igor Molybog and Igor Tufanov and Ilias Leontiadis and Irina-Elena Veliche and Itai Gat and Jake Weissman and James Geboski and James Kohli and Janice Lam and Japhet Asher and Jean-Baptiste Gaya and Jeff Marcus and Jeff Tang and Jennifer Chan and Jenny Zhen and Jeremy Reizenstein and Jeremy Teboul and Jessica Zhong and Jian Jin and Jingyi Yang and Joe Cummings and Jon Carvill and Jon Shepard and Jonathan McPhie and Jonathan Torres and Josh Ginsburg and Junjie Wang and Kai Wu and Kam Hou U and Karan Saxena and Kartikay Khandelwal and Katayoun Zand and Kathy Matosich and Kaushik Veeraraghavan and Kelly Michelena and Keqian Li and Kiran Jagadeesh and Kun Huang and Kunal Chawla and Kyle Huang and Lailin Chen and Lakshya Garg and Lavender A and Leandro Silva and Lee Bell and Lei Zhang and Liangpeng Guo and Licheng Yu and Liron Moshkovich and Luca Wehrstedt and Madian Khabsa and Manav Avalani and Manish Bhatt and Martynas Mankus and Matan Hasson and Matthew Lennie and Matthias Reso and Maxim Groshev and Maxim Naumov and Maya Lathi and Meghan Keneally and Miao Liu and Michael L. Seltzer and Michal Valko and Michelle Restrepo and Mihir Patel and Mik Vyatskov and Mikayel Samvelyan and Mike Clark and Mike Macey and Mike Wang and Miquel Jubert Hermoso and Mo Metanat and Mohammad Rastegari and Munish Bansal and Nandhini Santhanam and Natascha Parks and Natasha White and Navyata Bawa and Nayan Singhal and Nick Egebo and Nicolas Usunier and Nikhil Mehta and Nikolay Pavlovich Laptev and Ning Dong and Norman Cheng and Oleg Chernoguz and Olivia Hart and Omkar Salpekar and Ozlem Kalinli and Parkin Kent and Parth Parekh and Paul Saab and Pavan Balaji and Pedro Rittner and Philip Bontrager and Pierre Roux and Piotr Dollar and Polina Zvyagina and Prashant Ratanchandani and Pritish Yuvraj and Qian Liang and Rachad Alao and Rachel Rodriguez and Rafi Ayub and Raghotham Murthy and Raghu Nayani and Rahul Mitra and Rangaprabhu Parthasarathy and Raymond Li and Rebekkah Hogan and Robin Battey and Rocky Wang and Russ Howes and Ruty Rinott and Sachin Mehta and Sachin Siby and Sai Jayesh Bondu and Samyak Datta and Sara Chugh and Sara Hunt and Sargun Dhillon and Sasha Sidorov and Satadru Pan and Saurabh Mahajan and Saurabh Verma and Seiji Yamamoto and Sharadh Ramaswamy and Shaun Lindsay and Shaun Lindsay and Sheng Feng and Shenghao Lin and Shengxin Cindy Zha and Shishir Patil and Shiva Shankar and Shuqiang Zhang and Shuqiang Zhang and Sinong Wang and Sneha Agarwal and Soji Sajuyigbe and Soumith Chintala and Stephanie Max and Stephen Chen and Steve Kehoe and Steve Satterfield and Sudarshan Govindaprasad and Sumit Gupta and Summer Deng and Sungmin Cho and Sunny Virk and Suraj Subramanian and Sy Choudhury and Sydney Goldman and Tal Remez and Tamar Glaser and Tamara Best and Thilo Koehler and Thomas Robinson and Tianhe Li and Tianjun Zhang and Tim Matthews and Timothy Chou and Tzook Shaked and Varun Vontimitta and Victoria Ajayi and Victoria Montanez and Vijai Mohan and Vinay Satish Kumar and Vishal Mangla and Vlad Ionescu and Vlad Poenaru and Vlad Tiberiu Mihailescu and Vladimir Ivanov and Wei Li and Wenchen Wang and Wenwen Jiang and Wes Bouaziz and Will Constable and Xiaocheng Tang and Xiaojian Wu and Xiaolan Wang and Xilun Wu and Xinbo Gao and Yaniv Kleinman and Yanjun Chen and Ye Hu and Ye Jia and Ye Qi and Yenda Li and Yilin Zhang and Ying Zhang and Yossi Adi and Youngjin Nam and Yu and Wang and Yu Zhao and Yuchen Hao and Yundi Qian and Yunlu Li and Yuzi He and Zach Rait and Zachary DeVito and Zef Rosnbrick and Zhaoduo Wen and Zhenyu Yang and Zhiwei Zhao and Zhiyu Ma},
    year = {2024},
    eprint = {2407.21783},
    archivePrefix = {arXiv},
    primaryClass = {cs.AI},
    url = {https://arxiv.org/abs/2407.21783}
}

@misc{qwen2025qwen25technicalreport,
    title = {{Qwen2.5} Technical Report},
    author = {An Yang and Baosong Yang and Beichen Zhang and Binyuan Hui and Bo Zheng and Bowen Yu and Chengyuan Li and Dayiheng Liu and Fei Huang and Haoran Wei and Huan Lin and Jian Yang and Jianhong Tu and Jianwei Zhang and Jianxin Yang and Jiaxi Yang and Jingren Zhou and Junyang Lin and Kai Dang and Keming Lu and Keqin Bao and Kexin Yang and Le Yu and Mei Li and Mingfeng Xue and Pei Zhang and Qin Zhu and Rui Men and Runji Lin and Tianhao Li and Tianyi Tang and Tingyu Xia and Xingzhang Ren and Xuancheng Ren and Yang Fan and Yang Su and Yichang Zhang and Yu Wan and Yuqiong Liu and Zeyu Cui and Zhenru Zhang and Zihan Qiu},
    year = {2024},
    eprint = {2412.15115},
    archivePrefix = {arXiv},
    primaryClass = {cs.CL},
    url = {https://arxiv.org/abs/2412.15115}
}

@misc{yang2025qwen3technicalreport,
    title = {{Qwen3} Technical Report},
    author = {An Yang and Anfeng Li and Baosong Yang and Beichen Zhang and Binyuan Hui and Bo Zheng and Bowen Yu and Chang Gao and Chengen Huang and Chenxu Lv and Chujie Zheng and Dayiheng Liu and Fan Zhou and Fei Huang and Feng Hu and Hao Ge and Haoran Wei and Huan Lin and Jialong Tang and Jian Yang and Jianhong Tu and Jianwei Zhang and Jianxin Yang and Jiaxi Yang and Jing Zhou and Jingren Zhou and Junyang Lin and Kai Dang and Keqin Bao and Kexin Yang and Le Yu and Lianghao Deng and Mei Li and Mingfeng Xue and Mingze Li and Pei Zhang and Peng Wang and Qin Zhu and Rui Men and Ruize Gao and Shixuan Liu and Shuang Luo and Tianhao Li and Tianyi Tang and Wenbiao Yin and Xingzhang Ren and Xinyu Wang and Xinyu Zhang and Xuancheng Ren and Yang Fan and Yang Su and Yichang Zhang and Yinger Zhang and Yu Wan and Yuqiong Liu and Zekun Wang and Zeyu Cui and Zhenru Zhang and Zhipeng Zhou and Zihan Qiu},
    year = {2025},
    eprint = {2505.09388},
    archivePrefix = {arXiv},
    primaryClass = {cs.CL},
    url = {https://arxiv.org/abs/2505.09388}
}

@misc{deepseekai2024deepseekv2strongeconomicalefficient,
    title = {{DeepSeek-V2}: A Strong, Economical, and Efficient Mixture-of-Experts Language Model},
    author = {DeepSeek-AI and Aixin Liu and Bei Feng and Bin Wang and Bingxuan Wang and Bo Liu and Chenggang Zhao and Chengqi Dengr and Chong Ruan and Damai Dai and Daya Guo and Dejian Yang and Deli Chen and Dongjie Ji and Erhang Li and Fangyun Lin and Fuli Luo and Guangbo Hao and Guanting Chen and Guowei Li and H. Zhang and Hanwei Xu and Hao Yang and Haowei Zhang and Honghui Ding and Huajian Xin and Huazuo Gao and Hui Li and Hui Qu and J. L. Cai and Jian Liang and Jianzhong Guo and Jiaqi Ni and Jiashi Li and Jin Chen and Jingyang Yuan and Junjie Qiu and Junxiao Song and Kai Dong and Kaige Gao and Kang Guan and Lean Wang and Lecong Zhang and Lei Xu and Leyi Xia and Liang Zhao and Liyue Zhang and Meng Li and Miaojun Wang and Mingchuan Zhang and Minghua Zhang and Minghui Tang and Mingming Li and Ning Tian and Panpan Huang and Peiyi Wang and Peng Zhang and Qihao Zhu and Qinyu Chen and Qiushi Du and R. J. Chen and R. L. Jin and Ruiqi Ge and Ruizhe Pan and Runxin Xu and Ruyi Chen and S. S. Li and Shanghao Lu and Shangyan Zhou and Shanhuang Chen and Shaoqing Wu and Shengfeng Ye and Shirong Ma and Shiyu Wang and Shuang Zhou and Shuiping Yu and Shunfeng Zhou and Size Zheng and T. Wang and Tian Pei and Tian Yuan and Tianyu Sun and W. L. Xiao and Wangding Zeng and Wei An and Wen Liu and Wenfeng Liang and Wenjun Gao and Wentao Zhang and X. Q. Li and Xiangyue Jin and Xianzu Wang and Xiao Bi and Xiaodong Liu and Xiaohan Wang and Xiaojin Shen and Xiaokang Chen and Xiaosha Chen and Xiaotao Nie and Xiaowen Sun and Xiaoxiang Wang and Xin Liu and Xin Xie and Xingkai Yu and Xinnan Song and Xinyi Zhou and Xinyu Yang and Xuan Lu and Xuecheng Su and Y. Wu and Y. K. Li and Y. X. Wei and Y. X. Zhu and Yanhong Xu and Yanping Huang and Yao Li and Yao Zhao and Yaofeng Sun and Yaohui Li and Yaohui Wang and Yi Zheng and Yichao Zhang and Yiliang Xiong and Yilong Zhao and Ying He and Ying Tang and Yishi Piao and Yixin Dong and Yixuan Tan and Yiyuan Liu and Yongji Wang and Yongqiang Guo and Yuchen Zhu and Yuduan Wang and Yuheng Zou and Yukun Zha and Yunxian Ma and Yuting Yan and Yuxiang You and Yuxuan Liu and Z. Z. Ren and Zehui Ren and Zhangli Sha and Zhe Fu and Zhen Huang and Zhen Zhang and Zhenda Xie and Zhewen Hao and Zhihong Shao and Zhiniu Wen and Zhipeng Xu and Zhongyu Zhang and Zhuoshu Li and Zihan Wang and Zihui Gu and Zilin Li and Ziwei Xie},
    year = {2024},
    eprint = {2405.04434},
    archivePrefix = {arXiv},
    primaryClass = {cs.CL},
    url = {https://arxiv.org/abs/2405.04434}
}

@inproceedings{zhang2025breakingfrozensubspaceimportance,
    title = {Breaking the Frozen Subspace: Importance Sampling for Low-Rank Optimization in {LLM} Pretraining},
    author = {Haochen Zhang and Junze Yin and Guanchu Wang and Zirui Liu and Lin F. Yang and Tianyi Zhang and Anshumali Shrivastava and Vladimir Braverman},
    booktitle = {Advances in Neural Information Processing Systems},
    volume = {38},
    year = {2025},
    doi = {10.52202/085713-0343},
    url = {https://proceedings.neurips.cc/paper_files/paper/2025/hash/0edd294b7632fc96903abfbf3b264fc1-Abstract-Conference.html}
}

@inproceedings{refael2025sumo,
    title = {{SUMO}: Subspace-Aware Moment-Orthogonalization for Accelerating Memory-Efficient {LLM} Training},
    author = {Yehonathan Refael and Guy Smorodinsky and Tom Tirer and Ofir Lindenbaum},
    booktitle = {Advances in Neural Information Processing Systems},
    volume = {38},
    year = {2025},
    doi = {10.52202/085713-4926},
    url = {https://proceedings.neurips.cc/paper_files/paper/2025/hash/d85a66edadd443ac2350e93c0287f4f9-Abstract-Conference.html}
}

@inproceedings{liu2025cola,
    title = {{CoLA}: Compute-Efficient Pre-Training of {LLM}s via Low-Rank Activation},
    author = {Ziyue Liu and Ruijie Zhang and Zhengyang Wang and Mingsong Yan and Zi Yang and Paul D. Hovland and Bogdan Nicolae and Franck Cappello and Sui Tang and Zheng Zhang},
    booktitle = {Proceedings of the 2025 Conference on Empirical Methods in Natural Language Processing},
    pages = {4627--4645},
    publisher = {Association for Computational Linguistics},
    year = {2025},
    doi = {10.18653/v1/2025.emnlp-main.230},
    url = {https://aclanthology.org/2025.emnlp-main.230/}
}

@inproceedings{zhu2025apollo,
    title = {{APOLLO}: {SGD}-like Memory, {AdamW}-level Performance},
    author = {Hanqing Zhu and Zhenyu Zhang and Wenyan Cong and Xi Liu and Sem Park and Vikas Chandra and Bo Long and David Z. Pan and Zhangyang Wang and Jinwon Lee},
    booktitle = {Proceedings of Machine Learning and Systems},
    volume = {7},
    year = {2025},
    url = {https://proceedings.mlsys.org/paper_files/paper/2025/hash/437bc4ccafd3fc6d4289bd10940be42b-Abstract-Conference.html}
}

@inproceedings{loshchilov2019adamw,
    title = {Decoupled Weight Decay Regularization},
    author = {Ilya Loshchilov and Frank Hutter},
    booktitle = {International Conference on Learning Representations},
    year = {2019},
    url = {https://arxiv.org/abs/1711.05101}
}

@inproceedings{salimans2016weightnorm,
    title = {Weight Normalization: A Simple Reparameterization to Accelerate Training of Deep Neural Networks},
    author = {Tim Salimans and Diederik P. Kingma},
    booktitle = {Advances in Neural Information Processing Systems},
    volume = {29},
    year = {2016},
    url = {https://proceedings.neurips.cc/paper_files/paper/2016/hash/ed265bc903a5a097f61d3ec064d96d2e-Abstract.html}
}

@inproceedings{liu2024dora,
    title = {{DoRA}: Weight-Decomposed Low-Rank Adaptation},
    author = {Shih-Yang Liu and Chien-Yi Wang and Hongxu Yin and Pavlo Molchanov and Yu-Chiang Frank Wang and Kwang-Ting Cheng and Min-Hung Chen},
    booktitle = {Proceedings of the 41st International Conference on Machine Learning},
    volume = {235},
    pages = {32100--32121},
    series = {Proceedings of Machine Learning Research},
    publisher = {PMLR},
    year = {2024},
    url = {https://proceedings.mlr.press/v235/liu24bn.html}
}

@inproceedings{shazeer2018adafactor,
    title = {Adafactor: Adaptive Learning Rates with Sublinear Memory Cost},
    author = {Noam Shazeer and Mitchell Stern},
    booktitle = {Proceedings of the 35th International Conference on Machine Learning},
    volume = {80},
    pages = {4596--4604},
    series = {Proceedings of Machine Learning Research},
    publisher = {PMLR},
    year = {2018},
    url = {https://proceedings.mlr.press/v80/shazeer18a.html}
}

@misc{wang2025vlorp,
    title = {Memory-Efficient {LLM} Training by Various-Grained Low-Rank Projection of Gradients},
    author = {Yezhen Wang and Zhouhao Yang and Brian K Chen and Fanyi Pu and Bo Li and Tianyu Gao and Kenji Kawaguchi},
    year = {2025},
    eprint = {2505.01744},
    archivePrefix = {arXiv},
    primaryClass = {cs.LG},
    url = {https://arxiv.org/abs/2505.01744}
}

@inproceedings{li2024lorap,
    title = {{LoRAP}: Transformer Sub-Layers Deserve Differentiated Structured Compression for Large Language Models},
    author = {Guangyan Li and Yongqiang Tang and Wensheng Zhang},
    booktitle = {Proceedings of the 41st International Conference on Machine Learning},
    volume = {235},
    pages = {28657--28672},
    series = {Proceedings of Machine Learning Research},
    publisher = {PMLR},
    year = {2024},
    url = {https://proceedings.mlr.press/v235/li24bi.html}
}

@inproceedings{wong2025a3,
    title = {{A3} : an Analytical Low-Rank Approximation Framework for Attention},
    author = {Jeffrey T. H. Wong and Cheng Zhang and Xinye Cao and Pedro Gimenes and Christos-Savvas Bouganis and George Anthony Constantinides and Wayne Luk and Yiren Zhao},
    booktitle = {Proceedings of the 43rd International Conference on Machine Learning},
    series    = {Proceedings of Machine Learning Research},
    volume    = {306},
    year      = {2026},
    publisher = {PMLR},
    url       = {https://openreview.net/forum?id=aeeo8ZAftQ}
}

@inproceedings{zhang2023adalora,
    title = {{AdaLoRA}: Adaptive Budget Allocation for Parameter-Efficient Fine-Tuning},
    author = {Qingru Zhang and Minshuo Chen and Alexander Bukharin and Pengcheng He and Yu Cheng and Weizhu Chen and Tuo Zhao},
    booktitle = {International Conference on Learning Representations},
    year = {2023},
    url = {https://openreview.net/forum?id=lq62uWRJjiY}
}

@inproceedings{liu2024alora,
    title = {{ALoRA}: Allocating Low-Rank Adaptation for Fine-tuning Large Language Models},
    author = {Zequan Liu and Jiawen Lyn and Wei Zhu and Xing Tian and Yvette Graham},
    booktitle = {Proceedings of the 2024 Conference of the North American Chapter of the Association for Computational Linguistics: Human Language Technologies (Volume 1: Long Papers)},
    pages = {622--641},
    publisher = {Association for Computational Linguistics},
    year = {2024},
    doi = {10.18653/v1/2024.naacl-long.35},
    url = {https://aclanthology.org/2024.naacl-long.35/}
}

@misc{xv2025ara,
    title = {{ARA}: Adaptive Rank Allocation for Efficient Large Language Model {SVD} Compression},
    author = {Lin Xv and Jingsheng Gao and Xian Gao and Ting Liu and Yuzhuo Fu},
    year = {2025},
    eprint = {2510.19389},
    archivePrefix = {arXiv},
    primaryClass = {cs.LG},
    url = {https://arxiv.org/abs/2510.19389}
}

@article{raffel2023exploringlimitstransferlearning,
    title = {Exploring the Limits of Transfer Learning with a Unified Text-to-Text Transformer},
    author = {Colin Raffel and Noam Shazeer and Adam Roberts and Katherine Lee and Sharan Narang and Michael Matena and Yanqi Zhou and Wei Li and Peter J. Liu},
    journal = {Journal of Machine Learning Research},
    volume = {21},
    number = {140},
    pages = {1--67},
    year = {2020},
    url = {https://jmlr.org/papers/v21/20-074.html}
}

@inproceedings{zellers2019hellaswag,
    title = {{HellaSwag}: Can a Machine Really Finish Your Sentence?},
    author = {Rowan Zellers and Ari Holtzman and Yonatan Bisk and Ali Farhadi and Yejin Choi},
    booktitle = {Proceedings of the 57th Annual Meeting of the Association for Computational Linguistics},
    pages = {4791--4800},
    publisher = {Association for Computational Linguistics},
    year = {2019},
    doi = {10.18653/v1/P19-1472},
    url = {https://aclanthology.org/P19-1472/}
}

@article{sakaguchi2020winogrande,
    title = {{WinoGrande}: An Adversarial Winograd Schema Challenge at Scale},
    author = {Keisuke Sakaguchi and Ronan Le Bras and Chandra Bhagavatula and Yejin Choi},
    journal = {Proceedings of the AAAI Conference on Artificial Intelligence},
    volume = {34},
    number = {05},
    pages = {8732--8740},
    year = {2020},
    doi = {10.1609/aaai.v34i05.6399},
    url = {https://ojs.aaai.org/index.php/AAAI/article/view/6399}
}

@article{bisk2020piqa,
    title = {{PIQA}: Reasoning about Physical Commonsense in Natural Language},
    author = {Yonatan Bisk and Rowan Zellers and Ronan Le Bras and Jianfeng Gao and Yejin Choi},
    journal = {Proceedings of the AAAI Conference on Artificial Intelligence},
    volume = {34},
    number = {05},
    pages = {7432--7439},
    year = {2020},
    doi = {10.1609/aaai.v34i05.6239},
    url = {https://ojs.aaai.org/index.php/AAAI/article/view/6239}
}

@inproceedings{clark2019boolq,
    title = {{BoolQ}: Exploring the Surprising Difficulty of Natural Yes/No Questions},
    author = {Christopher Clark and Kenton Lee and Ming-Wei Chang and Tom Kwiatkowski and Michael Collins and Kristina Toutanova},
    booktitle = {Proceedings of the 2019 Conference of the North American Chapter of the Association for Computational Linguistics: Human Language Technologies, Volume 1 (Long and Short Papers)},
    pages = {2924--2936},
    publisher = {Association for Computational Linguistics},
    year = {2019},
    doi = {10.18653/v1/N19-1300},
    url = {https://aclanthology.org/N19-1300/}
}

@misc{clark2018think,
    title = {Think you have Solved Question Answering? Try {ARC}, the {AI2} Reasoning Challenge},
    author = {Peter Clark and Isaac Cowhey and Oren Etzioni and Tushar Khot and Ashish Sabharwal and Carissa Schoenick and Oyvind Tafjord},
    year = {2018},
    eprint = {1803.05457},
    archivePrefix = {arXiv},
    primaryClass = {cs.AI},
    url = {https://arxiv.org/abs/1803.05457}
}

@inproceedings{lin2022truthfulqa,
    title = {{TruthfulQA}: Measuring How Models Mimic Human Falsehoods},
    author = {Stephanie Lin and Jacob Hilton and Owain Evans},
    booktitle = {Proceedings of the 60th Annual Meeting of the Association for Computational Linguistics (Volume 1: Long Papers)},
    pages = {3214--3252},
    publisher = {Association for Computational Linguistics},
    year = {2022},
    doi = {10.18653/v1/2022.acl-long.229},
    url = {https://aclanthology.org/2022.acl-long.229/}
}

@inproceedings{merity2017pointer,
    title = {Pointer Sentinel Mixture Models},
    author = {Stephen Merity and Caiming Xiong and James Bradbury and Richard Socher},
    booktitle = {International Conference on Learning Representations},
    year = {2017},
    url = {https://openreview.net/forum?id=Byj72udxe}
}

% ============================================================================
% ============================================================================
% ⬆️⬆️⬆️⬆️ REFERENCE & ACKNOWLEDGMENTS ⬆️⬆️⬆️⬆️
% ============================================================================
% ============================================================================

\clearpage

% ============================================================================
% ============================================================================
% ⬇️⬇️⬇️⬇️ APPENDIX ⬇️⬇️⬇️⬇️
% ============================================================================
% ============================================================================

\appendix

% ============================================================================
% ============================================================================
% ⬇️⬇️⬇️⬇️ APPENDIX A ⬇️⬇️⬇️⬇️
% ============================================================================
% ============================================================================

\section{Positioning of MoARa}
\label{app_sec:positioning}
\label{appx:positioning}

Table~\ref{tab:positioning} summarizes how MoARa positions relative to the closest concurrent methods along five axes: architecture preservation, optimizer-state memory regime, module-aware rank allocation, magnitude--direction decomposition, and primary design objective. The memory regime is classified as \emph{low-rank} (compression to an $r$-dimensional gradient subspace), \emph{SGD-level} (per-channel/per-tensor scalars via random projection rather than SVD), or \emph{indirect} (memory reduced primarily through architectural changes). For axis context we also include DoRA~\citep{liu2024dora} and AdaLoRA~\citep{zhang2023adalora}, two fine-tuning methods.

\begin{table*}[t]
    \centering
    \caption{\textbf{Positioning of MoARa relative to closest concurrent works.}}
    \label{tab:positioning}
    \small
    \setlength{\tabcolsep}{3pt}
    \renewcommand{\arraystretch}{1.15}
    
    \begin{tabular*}{\textwidth}{@{\extracolsep{\fill}}lccccc@{}}
    \toprule
    \textbf{Method} & \makecell{Arch.\\preserved?} & \makecell{Opt.-state\\regime} & \makecell{Module-aware\\allocation?} & \makecell{Mag.--dir.\\decomp.?} & \makecell{Primary\\objective} \\
    \midrule
    GaLore        & \checkmark & low-rank   & ${\times}$ & ${\times}$ & memory reduction \\
    LDAdam        & \checkmark & low-rank   & ${\times}$ & ${\times}$ & adaptive low-D stats \\
    OSD           & \checkmark & low-rank   & ${\times}$ & ${\times}$ & online subspace update \\
    SubTrack++    & \checkmark & low-rank   & ${\times}$ & ${\times}$ & Grassmannian tracking \\
    Q-GaLore      & \checkmark & low-rank   & ${\times}$ & ${\times}$ & quant.\ + memory \\
    Fira          & \checkmark & low-rank   & ${\times}$ & ${\times}$ & full-rank update under low-rank state \\
    APOLLO        & \checkmark & SGD-level  & ${\times}$ & ${\times}$ & extreme memory reduction \\
    CoLA          & ${\times}$ & indirect & ${\times}$ & ${\times}$ & architectural compression \\
    DoRA          & \checkmark & (fine-tune) & ${\times}$ & \checkmark (weight-side) & fine-tuning capacity \\
    AdaLoRA       & \checkmark & (fine-tune) & \checkmark (LoRA budget) & ${\times}$ & fine-tuning budget \\
    \midrule
    \textbf{MoARa (ours)} & \checkmark & low-rank & \checkmark (rank budget) & \checkmark (gradient-side) & \makecell{convergence speed\\under a fixed projection-rank budget} \\
    \bottomrule
    \end{tabular*}
\end{table*}

% ============================================================================
% ============================================================================
% ⬆️⬆️⬆️⬆️ APPENDIX A ⬆️⬆️⬆️⬆️
% ============================================================================
% ============================================================================

% ============================================================================
% ============================================================================
% ⬇️⬇️⬇️⬇️ APPENDIX B ⬇️⬇️⬇️⬇️
% ============================================================================
% ============================================================================

\section{Block-wise Decomposition: Method Details and Analysis}
\label{app_sec:block_decomposition}

This section provides the implementation details, justification, and theoretical analysis underlying the block-wise magnitude--direction decomposition introduced in \S\ref{sec:block_decomposition} of the main text.

\subsection{Implementation of Block-wise Decomposition on Flattened Tensors}
\label{sec:appendix_flatten}

For implementation, we apply block-wise decomposition by preserving the first axis and flattening all remaining axes.
This gives a unified rule for all target linear layers.
In other words, we keep the row axis unchanged and apply blocking only along the column side.

For \(G \in \mathbb{R}^{s_0 \times s_1 \times \cdots \times s_{r-1}}\), we define
\[
{G}_{\mathrm{flat}}=\mathrm{reshape}(G,m,n),\;\;
m=s_0,\;\;
n=\prod_{\ell=1}^{r-1} s_\ell.
\]

If $B\nmid n$, we right-zero-pad the flattened column dimension to
\[
n_{\mathrm{pad}}
=
B\left\lceil\frac{n}{B}\right\rceil
\]
and denote the padded tensor by
\[
G_{\mathrm{pad}}
=
\operatorname{PadRight}
\left(
G_{\mathrm{flat}},
n_{\mathrm{pad}}-n
\right)
\in\mathbb{R}^{m\times n_{\mathrm{pad}}}.
\]
The padding is used only to complete the final block.

For block size $B$, each padded row is partitioned into
$k=n_{\mathrm{pad}}/B$ contiguous blocks:
\[
g_{i,b}
=
(G_{\mathrm{pad}})_{i,\,bB:(b+1)B}
\in\mathbb{R}^{B},
\ \
b=0,\ldots,k-1.
\]

Block magnitude and direction are defined as
\[
m_{i,b}=\|g_{i,b}\|_2,\qquad
v_{i,b}=\frac{g_{i,b}}{\|g_{i,b}\|_2+\epsilon},
\]
where \(\epsilon>0\) is a small constant for numerical stability.

In our implementation, \(\epsilon\) appears only in the normalization denominator.
At recomposition, we multiply magnitude and direction directly:
\[
\hat g_{i,b}=m_{i,b}v_{i,b}
=\frac{\|g_{i,b}\|_2}{\|g_{i,b}\|_2+\epsilon}\,g_{i,b}.
\]

Hence \(\hat g_{i,b}\neq g_{i,b}\) for finite \(\epsilon\), and exact equality is recovered as \(\epsilon\to 0\).
We use \(\epsilon = 10^{-12}\) in all experiments.
Because the padded entries are zero, padding does not change the $L^2$ magnitude of the final partial block.
After forming the direction tensor, we discard the padded columns and reshape the remaining entries back to the original gradient shape.
During recomposition, the direction update is temporarily padded, combined block-wise with the magnitude update, and cropped back to the original $n$ columns.

\subsection{Module-wise Interpretation Across Transformer Layers}
\label{sec:appendix_modules}

The same block rule is used on all target linear layers:
\[
\mathcal{W}_{\text{target}}=
\{W_q,W_k,W_v,W_o,W_{\text{up}},W_{\text{gate}},W_{\text{down}}\}.
\]

For each layer \(\ell\), with \(G^{(\ell)}=\nabla W^{(\ell)}\), we preserve the row axis and block along columns:
\[
g_{i,b}^{(\ell)}=G^{(\ell)}_{i,\;bB:(b+1)B}.
\]

Although the blocking rule is the same, its interpretation depends on the module.
In decoder-only transformers, different column-side channel groups can play different roles.
Because of this, the same coarse normalization range may not be equally suitable for all modules.

Under this view,
(i) for \(q/k/v\), blocking acts along the representation axis within each row;
(ii) for \(W_o\in\mathbb{R}^{d_{\text{hidden}}\times(Hd_{\text{head}})}\), blocking acts over concatenated head-channel columns.

This module-wise view helps explain why coarse normalization can be especially harmful in transformer blocks with heterogeneous column-side structure.

\subsection{Detailed Analysis: Magnitude-Factor Preservation and the Block Size Trade-off}
\label{appx:theory}

This appendix provides the formal magnitude-factor preservation result referenced in \S\ref{sec:theoretical_foundations} and the detailed small-$B$/large-$B$ derivation underlying the block-size trade-off discussed there.

\subsubsection{Setup}
\label{appx:theory_setup}

For notational simplicity, the derivation below assumes $B\mid n$; the same definitions apply to the shorter final block when $B\nmid n$.
As in \S\ref{sec:block_decomposition}, we partition the column index set $\{0, \ldots, n-1\}$ into $K = n / B$ contiguous blocks of size $B$, indexed by $b \in \{0, \ldots, K-1\}$.
For each row $i$ and block $b$, let $G_{i,b}\in\mathbb{R}^{B}$ denote the corresponding sub-vector of $G$.
For the exact identities used in the analysis below, we consider the $\epsilon\rightarrow0$ limit of the implementation described in Appendix~\ref{sec:appendix_flatten}.
For each nonzero block, we define
\[
m_{i,b}=\|G_{i,b}\|_2,
\qquad
v_{i,b}=\frac{G_{i,b}}{m_{i,b}},
\]
so that
\[
G_{i,b}=m_{i,b}v_{i,b},
\qquad
\|v_{i,b}\|_2=1
\]
exactly.
For a zero block, we define $v_{i,b}=0$.
The implementation uses $\epsilon=10^{-12}$ in the normalization denominator only for numerical stability.
MoARa applies low-rank projection only to the direction component $V \in \mathbb{R}^{m \times n}$ (where each row block $V_{i,b} = v_{i,b}$) and treats the magnitude component $M \in \mathbb{R}^{m \times K}$ (where $M_{i,b} = m_{i,b}$) as a separate optimizer state. The projected direction is $\tilde{V} = P_V P_V^\top V$, where $P_V \in \mathbb{R}^{m \times r}$ is obtained from the SVD of $V$. The reconstructed gradient is $\tilde{G}_{i,b} = m_{i,b} \cdot \tilde{v}_{i,b}$, where $\tilde{v}_{i,b} \in \mathbb{R}^B$ denotes the corresponding row block of $\tilde{V}$. 

\subsubsection{Magnitude-Factor Preservation}
\label{appx:theory_magnitude}

Recall that pure low-rank projection acts on the full gradient matrix as $G_{\mathrm{proj}}=PP^\top G$.
We formalize a property specific to MoARa's decomposition: the pre-projection block-magnitude factor is carried outside the direction projection, so it does not contribute to block-wise cosine misalignment.

\begin{proposition}
\label{prop:magnitude_preservation}
For the alignment diagnostic,
let
\[
\widetilde{V}=P_VP_V^\top V,
\]
and let $\widetilde{v}_{i,b}$ denote the row block of $\widetilde{V}$
corresponding to row $i$ and block $b$.
MoARa reconstructs each gradient block as
\[
\widetilde{G}_{i,b}
=
m_{i,b}\widetilde{v}_{i,b}.
\]
Thus, the pre-projection magnitude factor $m_{i,b}$ is carried outside the direction projection and reintroduced unchanged.
Whenever $G_{i,b}$ and $\widetilde{G}_{i,b}$ are both nonzero,
\[
\cos(G_{i,b},\widetilde{G}_{i,b})
=
\cos(v_{i,b},\widetilde{v}_{i,b}).
\]
\end{proposition}

\begin{proof}
For a nonzero block,
$G_{i,b}=m_{i,b}v_{i,b}$ and
$\widetilde{G}_{i,b}=m_{i,b}\widetilde{v}_{i,b}$.
Therefore,
\[
\begin{aligned}
\cos(G_{i,b},\widetilde{G}_{i,b})
&= \frac{
m_{i,b}^{2}
\langle v_{i,b},\widetilde{v}_{i,b}\rangle
}{
m_{i,b}^{2}
\|v_{i,b}\|_2
\|\widetilde{v}_{i,b}\|_2
} \\
&= \cos(v_{i,b},\widetilde{v}_{i,b}).
\end{aligned}
\]
Hence, the scalar magnitude factor does not contribute to the block-wise cosine misalignment; any such misalignment is determined by the projected direction branch.
\end{proof}

\paragraph{Connection to weight normalization.}
Proposition~\ref{prop:magnitude_preservation} is the gradient-side analogue of the conditioning argument made by Salimans and Kingma~\citep{salimans2016weightnorm} for weight normalization, who showed that reparameterizing a parameter $w$ as $w = g \cdot v / \|v\|$ improves the conditioning of stochastic gradient descent by decoupling scale and direction. DoRA~\citep{liu2024dora} extends this principle to parameter-efficient fine-tuning by decomposing pretrained weights into magnitude and direction. MoARa adapts this principle to the gradient side in the context of low-rank pretraining: the motivation is not optimization conditioning in itself, but separating block-scale information from projection-induced cosine distortion, complementing the rank allocation strategy of \S\ref{sec:rank_allocation}.

\subsubsection{Block Size Trade-off: Detailed Derivation}
\label{appx:theory_block_size}

\paragraph{Small-$B$ regime (statistical noise).}
When $B$ is small, each magnitude $m_{i,b}$ is computed over only a few coordinates. 
In the extreme $B = 1$, we have $m_{i,j} = |G_{i,j}|$ and $v_{i,j} = \operatorname{sign}(G_{i,j})$, with $v_{i,j}=0$ when $G_{i,j}=0$.
Thus, $V\in\{-1,0,+1\}^{m\times n}$, and the relative magnitude structure of $G$ is absorbed into $M$. 
The directional branch $V$ then operates on a sign-valued matrix that no longer contains the original coordinate-wise magnitude variation.
This extreme normalization can reduce the information available to the directional branch.
More generally, for small $B$, $m_{i,b}$ inherits the noise of small-sample $L^2$ norm estimates, becoming progressively less reliable as a summary of the block's energy. 
Under approximate within-row stationarity and weak dependence, let
\[
X_{i,j}=G_{i,j}^{2},
\qquad
\mu_2=\mathbb{E}[X_{i,j}]>0,
\]
and assume $\operatorname{Var}(X_{i,j})<\infty$.
For the block energy
\[
S_{i,b}
=
\|G_{i,b}\|_2^2
=
\sum_{j\in b}X_{i,j},
\]
we then have
\[
\mathbb{E}[S_{i,b}]
\approx
B\mu_2,
\qquad
\operatorname{Var}(S_{i,b})
=
O(B).
\]
Applying a first-order delta-method approximation to
$m_{i,b}=\sqrt{S_{i,b}}$ gives
\[
\mathbb{E}[m_{i,b}]
=
\Theta(\sqrt{B}),
\qquad
\operatorname{Std}[m_{i,b}]
=
O(1),
\]
and therefore
\[
\frac{\operatorname{Std}[m_{i,b}]}
{\mathbb{E}[m_{i,b}]}
=
O(B^{-1/2}).
\]
Thus, the relative variability of the block-magnitude estimate decreases as $B$ grows. 

\paragraph{Optimizer-dynamics view at $B = 1$.}
At $B=1$, all coordinate-wise magnitude variation is carried by $M$, while the direction optimizer receives only a projection of the sign-valued direction signal.
This extreme separation between the two branches provides one possible explanation for the degraded $B = 1$ result in Figure~\ref{fig:block_size_sweep}.

\paragraph{Large-$B$ regime (column-axis subspace coupling).}
\label{appx:column_axis_coupling}
When $B$ is large, $m_{i,b}$ aggregates over heterogeneous coordinates that may span multiple attention heads. The unit direction $v_{i,b} = G_{i,b} / m_{i,b}$ then exhibits an attenuation phenomenon we call \emph{column-axis subspace coupling}: coordinates with relatively small magnitude within $G_{i,b}$ are scaled down by a factor proportional to the dominant coordinates' magnitude. Formally, partitioning $G_{i,b}$ into a high-magnitude subset $\mathcal{A}$ and a low-magnitude subset $\mathcal{C}$ (with $|G_{i,a}| \gg |G_{i,c}|$ coordinate-wise for $a \in \mathcal{A}$, $c \in \mathcal{C}$), the corresponding directional sub-norm satisfies
\begin{equation*}
\begin{aligned}
\|v_{i,\mathcal{C}}\|_2
&=
\frac{\|G_{i,\mathcal{C}}\|_2}
{\sqrt{\|G_{i,\mathcal{A}}\|_2^2 + \|G_{i,\mathcal{C}}\|_2^2}} \\
&\leq
\frac{\|G_{i,\mathcal{C}}\|_2}
{\|G_{i,\mathcal{A}}\|_2},
\end{aligned}
\end{equation*}
which can become arbitrarily small as the magnitude ratio grows. 
In the limit $B = n$ (rowwise normalization), each $v_{i,1}$ spans an entire row and the dominant coordinate of $G_{i,:}$ suppresses information from all subdominant directions, explaining why an overly coarse decomposition suppresses weaker local signals. 
This effect is amplified when a block contains a few unusually large coordinates, because the shared magnitude scale is then dominated by those entries.

\paragraph{Balanced regime.}
The two effects intersect in an intermediate regime: small-$B$ noise diminishes as $B$ increases, while large-$B$ coupling grows with $B$.
For Transformer gradients, the attention-head dimension provides a natural architecture-informed reference scale for choosing the block granularity.
This suggests selecting $B$ in the neighborhood of $d_{\mathrm{head}}$, without requiring exact head alignment.
Appendix~\ref{appx:block_size_sensitivity} reports empirical evidence consistent with this choice, observing a plateau of near-optimal performance across $B\in[d_{\mathrm{head}}/2,\,2d_{\mathrm{head}}]$ on Llama~2 350M ($d_{\mathrm{head}}=64$).

\begin{figure}[t]
    \centering
    \includegraphics[width=\linewidth]{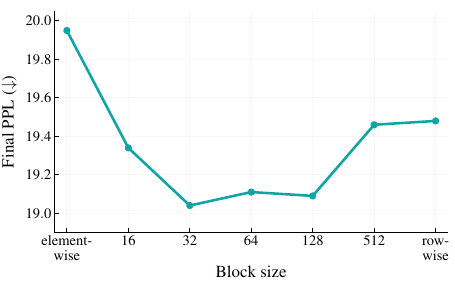}
    \caption{\textbf{Block size sensitivity sweep on Llama 2 350M 60k steps.}
    Final validation PPL as a function of $B \in \{1 \ (\text{element}), 16, 32, 64, 128, 512, 1024 \ (\text{row}) \}$. The U-shape is consistent with the statistical--coupling trade-off analyzed in Appendix~\ref{appx:theory_block_size}. 
    The plateau region $B \in \{32, 64, 128\}$ covers the head dimension of the Llama 2 scales we evaluate.}
    \label{fig:block_size_sweep}
\end{figure}

\subsection{Block Size Sensitivity Sweep}
\label{appx:block_size_sensitivity}

We characterize the sensitivity of block-wise decomposition to the block size parameter $B$ by sweeping $B \in \{1, 16, 32, 64, 128, 512, 1024\}$ on Llama 2~\citep{touvron2023llama2openfoundation} 350M while keeping all other settings identical to the full MoARa configuration. Figure~\ref{fig:block_size_sweep} reports final validation perplexity at 60k steps as a function of $B$. The curve is U-shaped: both endpoints --- elementwise ($B = 1$, PPL 19.95) and rowwise ($B = 1024$, PPL 19.48) --- give clearly higher perplexity than the interior, while a region of small variation spans $B = 32$ through $B = 128$ (PPL 19.04, 19.11, 19.09 respectively; all within 0.07 PPL of the best).

The U-shaped trend is consistent with the statistical--coupling trade-off analyzed in Appendix~\ref{appx:theory_block_size}.
At $B=1$, the direction signal becomes sign-valued and loses coordinate-wise magnitude variation, which provides
a possible qualitative explanation for the degraded trajectory.
At the largest tested block size, the coarse block-wise magnitude statistics are less able to track heterogeneous
local scale structure.
The head dimension of Llama 2 350M sits inside the plateau region $B \in \{32, 64, 128\}$, providing the architecture-aware default used throughout our experiments.

% ============================================================================
% ============================================================================
% ⬆️⬆️⬆️⬆️ APPENDIX B ⬆️⬆️⬆️⬆️
% ============================================================================
% ============================================================================

% ============================================================================
% ============================================================================
% ⬇️⬇️⬇️⬇️ APPENDIX C ⬇️⬇️⬇️⬇️
% ============================================================================
% ============================================================================

\section{Module-aware Rank Allocation: Procedure and Analysis}
\label{app_sec:rank_allocation}

This section provides the algorithmic procedure and supporting evidence for the module-aware projection-rank allocation introduced in \S\ref{sec:rank_allocation}.

\subsection{Profiling-Based Rank Allocation}
\label{appx:profiling}

Our rank reallocation is based on a profiling step conducted before the target MoARa training run. We first run the standard GaLore~\citep{zhao2024galore} configuration for $N_p$ steps with uniform rank $r_{\mathrm{base}}$ across all modules, logging the full gradient $G_t(W)$ and the reconstructed gradient $G_{\mathrm{recon},t}(W) = \mathrm{Rep}_B(M) \odot (P_\tau P_\tau^\top V)$ at sampled time points $t \in \mathcal{S}_p$. The module-wise alignment score $S_\tau$ is then computed as in Eq.~\eqref{eq:module_alignment}, and the resulting integer rank allocation is fixed for the entire training run, consistent with our static-allocation design.

% ============================================================================
% ============================================================================
% ⬇️⬇️⬇️⬇️ TABLE: PROFILING_WINDOW_SENSITIVITY ⬇️⬇️⬇️⬇️
% ============================================================================
% ============================================================================

\begin{table}[t]
    \centering
    \caption{\textbf{Short profiling runs match the full-trajectory assignment.}
    Module-wise ranks of $S_\tau$ are identical across 1k, 10k, and reference
    60k profiling runs on Llama~2 350M. Rank 1 denotes the highest cosine
    alignment and 7 the lowest. All module-wise ranks are identical across the three profiling settings.}
    \label{tab:profiling_length}
    \small
    \setlength{\tabcolsep}{3.8pt}
    \renewcommand{\arraystretch}{1.12}
    \begin{tabular*}{\columnwidth}{@{\extracolsep{\fill}}lcccc@{}}
    \toprule
    \textbf{Module}
    & \makecell{\textbf{1k}\\{\footnotesize $\sim$11 min}}
    & \makecell{\textbf{10k}\\{\footnotesize $\sim$72 min}}
    & \makecell{\textbf{60k (full)}\\{\footnotesize at 6k}}\\
    \midrule
    \attn{q}   & 5 (.899) & 5 (.892) & 5 (.892) \\
    \attn{k}   & 6 (.885) & 6 (.847) & 6 (.851) \\   
    \attn{v}   & 1 (.959) & 1 (.944) & 1 (.971) \\ 
    \attn{o}   & 2 (.932) & 2 (.934) & 2 (.960) \\  
    \mlp{up}   & 3 (.922) & 3 (.900) & 3 (.905) \\
    \mlp{gate} & 4 (.921) & 4 (.898) & 4 (.900) \\
    \mlp{down} & 7 (.819) & 7 (.721) & 7 (.734) \\
    \bottomrule
    \end{tabular*}
\end{table}

% ============================================================================
% ============================================================================
% ⬆️⬆️⬆️⬆️ TABLE: PROFILING_WINDOW_SENSITIVITY ⬆️⬆️⬆️⬆️
% ============================================================================
% ============================================================================

\subsection{Profiling protocol and intervals}

Three distinct intervals appear in our procedure and serve different purposes. The \emph{profiling stride} $x$ controls how often $S(W, t)$ is recorded during the baseline profiling run. The \emph{profiling window} $N_p$ is the total length of the baseline profiling run. The \emph{SVD update interval} $T_{\mathrm{SVD}}$ is how often the projection matrix $P$ is recomputed via TruncatedSVD during MoARa training. The profiling stride and window control the diagnostic only; only the SVD update interval is part of the MoARa training loop itself.

\subsection{Profiling cost and short-run sufficiency}
\label{appx:profiling_cost}

\paragraph{Profiling cost.}
The profiling phase reuses the standard GaLore training setup and adds negligible per-step overhead. 
The sampling interval $x$ in $\mathcal{S}_p = \{x, 2x, \ldots, N_p\}$ controls only the temporal resolution of the logged diagnostic and does not affect the resulting rank allocation. 
In our profiling runs we use $x = 50$; this choice is made for monitoring efficiency rather than for the allocation procedure itself.

\paragraph{The profiling phase is off the wall-clock-to-target path.}
The profiling run is conducted once before target training.
All raw wall-clock trajectories and training-only time-to-target measurements exclude this setup phase.
In \S\ref{sec:main_result}, we additionally report non-amortized and amortized time-to-target by accounting for the profiling cost separately.
The resulting donor/receiver assignment can be reused across subsequent target runs.

\paragraph{A short profiling run already suffices.}
We further find that this one-time cost itself can be made very small. On Llama~2~\citep{touvron2023llama2openfoundation} 350M, a 1k-step profiling run completes in approximately 11 minutes and a 10k-step run completes in approximately 72 minutes in our hardware settings. The budget-donor/receiver assignment obtained from either short run matches the assignment obtained from the full 60k-step profiling trajectory at its 6k checkpoint (Table~\ref{tab:profiling_length}). 
This indicates that, across short and long profiles, the donor/receiver assignment is established early in the training trajectory.
The observation is consistent with our claim in \S\ref{sec:analysis} that rank allocation operates on an architectural property of the Transformer rather than on optimizer-specific dynamics.

\paragraph{Amortization across scales.}
A 1k-step profiling run on Llama~2 350M corresponds to roughly 2.6\% of a single 350M target run (7.13~h, see Table~\ref{tab:cross_method}). 
The same donor/receiver assignment is reused at the 1B and 7B scales, while the numerical ranks follow each scale's $r_{\mathrm{base}}$.
The relative profiling overhead therefore drops further at those scales.
In practice, the profiling cost is small in absolute terms at 350M and effectively amortized when applied to larger target runs.
Accordingly, we distinguish the target-training time $T_{train}$ from the non-amortized time $T_{train}$ + $T_{profile}$. When the same assignment is reused across $K$ target runs, the corresponding amortized time-to-target is $T_{train}$ + $T_{profile}$ / $K$; \S\ref{sec:main_result} reports this accounting alongside the headline results.
% ============================================================================
% ============================================================================
% ⬇️⬇️⬇️⬇️ FIGURE: PROFILING_STRIDE_SENSITIVITY ⬇️⬇️⬇️⬇️
% ============================================================================
% ============================================================================

\begin{figure*}[t]
    \centering
    \includegraphics[width=\linewidth]{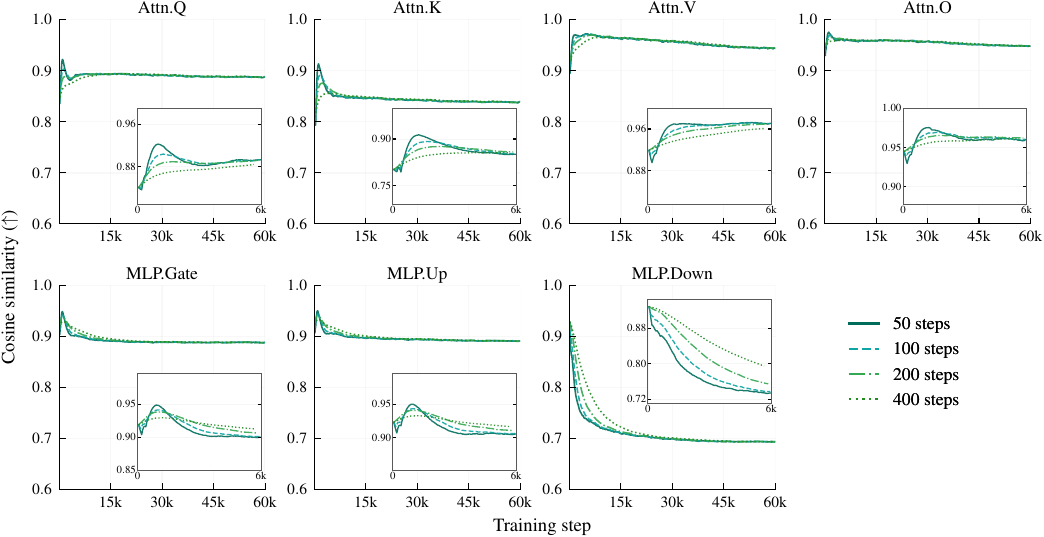}
    \caption{\textbf{Module-wise cosine similarity under varying profiling strides (Llama 2 350M, 60k-step standard GaLore).} Each panel plots the EMA of cosine similarity between full and reconstructed gradients for one projection module, sampled at strides of $50$, $100$, $200$, and $400$ steps. The inset zooms into the warmup region ($0$--$6$k). The donor/receiver assignment is insensitive to the choice of profiling stride.}
    \label{app_fig:profiling_stride_sensitivity}
\end{figure*}

% ============================================================================
% ============================================================================
% ⬆️⬆️⬆️⬆️ FIGURE: PROFILING_STRIDE_SENSITIVITY ⬆️⬆️⬆️⬆️
% ============================================================================
% ============================================================================

% ============================================================================
% ============================================================================
% ⬇️⬇️⬇️⬇️ TABLE: PROFILING_STRIDE_SENSITIVITY ⬇️⬇️⬇️⬇️
% ============================================================================
% ============================================================================

\begin{table*}[t]
    \centering
    \caption{\textbf{Module-wise cosine similarity rankings under different profiling intervals.}
    Values are EMA-smoothed cosine similarities. Modules are ordered from highest
    to lowest cosine similarity.}
    \label{tab:cosine_stride_ranking}
    \small
    \setlength{\tabcolsep}{3.8pt}
    \renewcommand{\arraystretch}{1.06}
    
    \begin{tabular*}{\textwidth}{@{\extracolsep{\fill}}cc*{7}{c}@{}}
    \toprule
    \textbf{Step} & \textbf{Interval}
    & \multicolumn{7}{c}{\textbf{Ranking by cosine similarity}} \\
    \cmidrule(l){3-9}
    & & \textbf{1st} & \textbf{2nd} & \textbf{3rd} & \textbf{4th}
      & \textbf{5th} & \textbf{6th} & \textbf{7th} \\
    \midrule
    
    \multirow{4}{*}{6k}
    & 50  & \makecell{V\\(0.9709)} & \makecell{O\\(0.9600)}
          & \makecell{Up\\(0.9052)} & \makecell{Gate\\(0.8997)}
          & \makecell{Q\\(0.8924)} & \makecell{K\\(0.8505)}
          & \makecell{Down\\(0.7339)} \\
    & 100 & \makecell{V\\(0.9713)} & \makecell{O\\(0.9610)}
          & \makecell{Up\\(0.9061)} & \makecell{Gate\\(0.9010)}
          & \makecell{Q\\(0.8924)} & \makecell{K\\(0.8530)}
          & \makecell{Down\\(0.7371)} \\
    & 200 & \makecell{V\\(0.9699)} & \makecell{O\\(0.9624)}
          & \makecell{Up\\(0.9103)} & \makecell{Gate\\(0.9059)}
          & \makecell{Q\\(0.8914)} & \makecell{K\\(0.8578)}
          & \makecell{Down\\(0.7533)} \\
    & 400 & \makecell{V\\(0.9623)} & \makecell{O\\(0.9614)}
          & \makecell{Up\\(0.9158)} & \makecell{Gate\\(0.9117)}
          & \makecell{Q\\(0.8844)} & \makecell{K\\(0.8550)}
          & \makecell{Down\\(0.7897)} \\
    
    \midrule
    
    \multirow{4}{*}{15k}
    & 50  & \makecell{V\\(0.9617)} & \makecell{O\\(0.9578)}
          & \makecell{Up\\(0.8959)} & \makecell{Q\\(0.8926)}
          & \makecell{Gate\\(0.8905)} & \makecell{K\\(0.8455)}
          & \makecell{Down\\(0.7091)} \\
    & 100 & \makecell{V\\(0.9627)} & \makecell{O\\(0.9583)}
          & \makecell{Up\\(0.8968)} & \makecell{Q\\(0.8928)}
          & \makecell{Gate\\(0.8912)} & \makecell{K\\(0.8459)}
          & \makecell{Down\\(0.7103)} \\
    & 200 & \makecell{V\\(0.9643)} & \makecell{O\\(0.9594)}
          & \makecell{Up\\(0.8986)} & \makecell{Q\\(0.8939)}
          & \makecell{Gate\\(0.8929)} & \makecell{K\\(0.8475)}
          & \makecell{Down\\(0.7131)} \\
    & 400 & \makecell{V\\(0.9643)} & \makecell{O\\(0.9591)}
          & \makecell{Up\\(0.9008)} & \makecell{Gate\\(0.8953)}
          & \makecell{Q\\(0.8935)} & \makecell{K\\(0.8489)}
          & \makecell{Down\\(0.7219)} \\
    
    \midrule
    
    \multirow{4}{*}{30k}
    & 50  & \makecell{V\\(0.9579)} & \makecell{O\\(0.9574)}
          & \makecell{Up\\(0.8942)} & \makecell{Q\\(0.8915)}
          & \makecell{Gate\\(0.8884)} & \makecell{K\\(0.8424)}
          & \makecell{Down\\(0.6992)} \\
    & 100 & \makecell{V\\(0.9590)} & \makecell{O\\(0.9581)}
          & \makecell{Up\\(0.8947)} & \makecell{Q\\(0.8918)}
          & \makecell{Gate\\(0.8890)} & \makecell{K\\(0.8428)}
          & \makecell{Down\\(0.6997)} \\
    & 200 & \makecell{V\\(0.9590)} & \makecell{O\\(0.9581)}
          & \makecell{Up\\(0.8949)} & \makecell{Q\\(0.8920)}
          & \makecell{Gate\\(0.8893)} & \makecell{K\\(0.8431)}
          & \makecell{Down\\(0.7002)} \\
    & 400 & \makecell{V\\(0.9592)} & \makecell{O\\(0.9579)}
          & \makecell{Up\\(0.8952)} & \makecell{Q\\(0.8926)}
          & \makecell{Gate\\(0.8895)} & \makecell{K\\(0.8440)}
          & \makecell{Down\\(0.7016)} \\
    
    \midrule
    
    \multirow{4}{*}{60k}
    & 50  & \makecell{O\\(0.9484)} & \makecell{V\\(0.9442)}
          & \makecell{Up\\(0.8909)} & \makecell{Q\\(0.8884)}
          & \makecell{Gate\\(0.8881)} & \makecell{K\\(0.8395)}
          & \makecell{Down\\(0.6936)} \\
    & 100 & \makecell{O\\(0.9487)} & \makecell{V\\(0.9444)}
          & \makecell{Up\\(0.8908)} & \makecell{Q\\(0.8880)}
          & \makecell{Gate\\(0.8879)} & \makecell{K\\(0.8392)}
          & \makecell{Down\\(0.6937)} \\
    & 200 & \makecell{O\\(0.9488)} & \makecell{V\\(0.9447)}
          & \makecell{Up\\(0.8908)} & \makecell{Q\\(0.8879)}
          & \makecell{Gate\\(0.8879)} & \makecell{K\\(0.8387)}
          & \makecell{Down\\(0.6936)} \\
    & 400 & \makecell{O\\(0.9484)} & \makecell{V\\(0.9439)}
          & \makecell{Up\\(0.8904)} & \makecell{Gate\\(0.8876)}
          & \makecell{Q\\(0.8874)} & \makecell{K\\(0.8380)}
          & \makecell{Down\\(0.6936)} \\
    
    \bottomrule
    \end{tabular*}
\end{table*}

% ============================================================================
% ============================================================================
% ⬆️⬆️⬆️⬆️ TABLE: PROFILING_STRIDE_SENSITIVITY ⬆️⬆️⬆️⬆️
% ============================================================================
% ============================================================================

\subsection{Profiling-stride and profiling-window sensitivity}

We test the dependence of the budget-donor/receiver assignment on both the profiling stride and the profiling window. We re-sample the recorded baseline GaLore trajectory at strides of $100$, $200$, and $400$ steps in addition to the default $50$ on Llama 2~\citep{touvron2023llama2openfoundation} 350M. We then examine the resulting EMA-smoothed cosine similarities at four checkpoints ($6$k, $15$k, $30$k, $60$k). Figure~\ref{app_fig:profiling_stride_sensitivity} reports the per-module trajectories. Table~\ref{tab:cosine_stride_ranking} reports the corresponding numerical rankings.

\paragraph{Donor and receiver positions are stable.}
$\mlp{down}$ is the lowest-scoring module in all sixteen (window, stride) combinations. In every case, it occupies rank 7 by a substantial margin. $\attn{k}$ occupies rank 6 in every case. $\attn{q}$ occupies rank 4 or 5 in every case. The gap $S_K - S_{\mathrm{Down}}$ stays in the range $[0.10, 0.15]$ across all combinations. $\mlp{down}$ is therefore clearly separated as the unique receiver.

\paragraph{Adjacent swaps occur only within the upper tier and stay tightly bounded.}
Some swaps appear in Table~\ref{tab:cosine_stride_ranking} among modules whose values lie within $0.01$ of one another. These all occur strictly within the upper-tier cluster ($\{\attn{v}, \attn{o}, \mlp{up}, \mlp{gate}, \attn{q}\}$). The largest observed within-tier difference is $0.011$ (between $\attn{v}$ and $\attn{o}$ at $6$k, stride $50$) and the typical difference is below $0.008$. These are about two orders of magnitude smaller than the donor--receiver boundary $S_K - S_{\mathrm{Down}} \in [0.10, 0.15]$. The upper-tier ordering is therefore sensitive to small perturbations in stride or window length, but the donor/receiver assignment used by Algorithm~\ref{alg:moara_full} is unaffected.

\paragraph{Alignment ordering is established within warmup.}
The inset in Figure~\ref{app_fig:profiling_stride_sensitivity} zooms into the warmup region ($0$--$6$k). The donor/receiver ordering ($\attn{q}$, $\attn{k}$ as donors; $\mlp{down}$ as receiver) is already present within this region and is preserved across the remainder of the $60$k-step run. This is consistent with the short-run sufficiency observed in Appendix~\ref{appx:profiling_cost} (a 1k-step profiling run, which lies entirely inside the warmup region, already recovers the full assignment). 
The pattern remains unchanged across the evaluated profiling strides of 50, 100, 200, and 400 steps.

\subsection{Final rank allocation}
\label{appx:final_rank_allocation}
The integer rank allocation $\{r_\tau\}_{\tau \in \mathcal{T}}$ used in our main experiments follows directly from our default configuration: $\mathcal{D} = \{\attn{q}, \attn{k}\}$, $\mathcal{R} = \{\mlp{down}\}$, $\Delta r = r_{\mathrm{base}}/2$. Per-module ranks then reduce to $r_Q = r_K = r_{\mathrm{base}}/2$ (donors), $r_V = r_O = r_{\mathrm{up}} = r_{\mathrm{gate}} = r_{\mathrm{base}}$ (uniform baseline), and $r_{\mathrm{down}} = 2\,r_{\mathrm{base}}$ (receiver). With $|\mathcal{D}| = 2$ and $|\mathcal{R}| = 1$, the total budget $\sum_\tau r_\tau = 7\,r_{\mathrm{base}}$ is preserved by construction, with no residual adjustment required for this choice of $\Delta r$, $|\mathcal{D}|$, and $|\mathcal{R}|$. Per-scale baseline ranks $r_{\mathrm{base}}$ follow the GaLore~\citep{zhao2024galore} convention at each model size and are reported in Appendix~\ref{appx:training_params}.

% ============================================================================
% ============================================================================
% ⬇️⬇️⬇️⬇️ FIGURE: RANK_SENSITIVITY ⬇️⬇️⬇️⬇️
% ============================================================================
% ============================================================================

\begin{figure}
    \centering
    \includegraphics[width=1\linewidth]{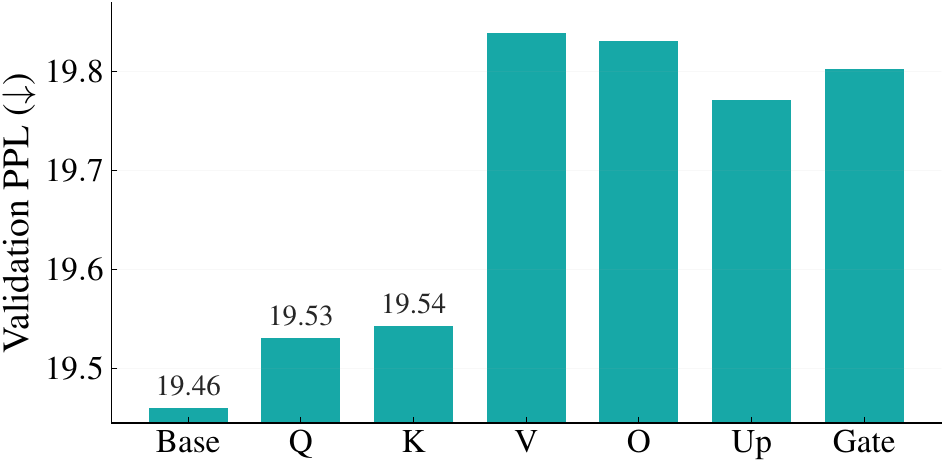}
    \caption{\textbf{Module-wise projection-rank sensitivity on Llama 2 350M.} We reduce the rank of one target module to \(r=128\) and keep the others at \(r=256\). \texttt{attn.q/k} are the most robust to rank reduction.}
    \label{app_fig:rank_sensitivity}
\end{figure}

% ============================================================================
% ============================================================================
% ⬆️⬆️⬆️⬆️ FIGURE: RANK_SENSITIVITY ⬆️⬆️⬆️⬆️
% ============================================================================
% ============================================================================

% ============================================================================
% ============================================================================
% ⬇️⬇️⬇️⬇️ TABLE: MODEL_HYPERPARAMETERS ⬇️⬇️⬇️⬇️
% ============================================================================
% ============================================================================

\begin{table*}[t]
    \centering
    \caption{\textbf{Model architectures and pre-training hyperparameters.}
    We report architecture details and the pre-training budget for each evaluated model.}
    \label{app_tab:model_hparams}
    \normalsize
    \renewcommand{\arraystretch}{1.12}
    \setlength{\tabcolsep}{4.8pt}
    \begin{tabular*}{\textwidth}{@{\extracolsep{\fill}}lcrrrrrrr@{}}
    \toprule
    \textbf{Model} & \textbf{Params}
    & \multicolumn{5}{c}{\textbf{Architecture}}
    & \multicolumn{2}{c}{\textbf{Pre-training}} \\
    \cmidrule(lr){3-7}
    \cmidrule(l){8-9}
    & &
    \textbf{Hidden} &
    \textbf{Interm.} &
    \textbf{Heads} &
    \textbf{Layers} &
    \textbf{KV} &
    \textbf{Steps} &
    \textbf{Tokens} \\
    \midrule
    \multicolumn{9}{@{}l}{\textit{Scaling study}} \\
    Llama 2     & 350M & 1024 &  2736 & 16 & 24 & 16 &  60K &  7.8B \\
    Llama 2     & 1B   & 2048 &  5461 & 32 & 24 & 16 & 100K & 13.1B \\
    Llama 2     & 7B   & 4096 & 11008 & 32 & 32 & 16 & 150K & 19.7B \\
    \addlinespace[2pt]
    \midrule
    \multicolumn{9}{@{}l}{\textit{Model-family comparison}} \\
    Llama 3.2   & 300M & 1024 &  2736 & 16 & 24 &  4 &  60K &  7.8B \\
    Qwen2.5     & 350M & 1024 &  2816 & 16 & 24 &  4 &  60K &  7.8B \\
    Qwen3       & 350M & 1024 &  2816 & 16 & 24 &  4 &  60K &  7.8B \\
    DeepSeek-V2 & 350M & 1024 &  2736 & 16 & 24 & 16 &  60K &  7.8B \\
    \bottomrule
    \end{tabular*}
\end{table*}

% ============================================================================
% ============================================================================
% ⬆️⬆️⬆️⬆️ TABLE: MODEL_HYPERPARAMETERS ⬆️⬆️⬆️⬆️
% ============================================================================
% ============================================================================

\subsection{Donor Suitability and Module-wise Sensitivity}
\label{appx:donor_sensitivity}

Projection sensitivity and donor suitability are related but not identical. A module with high alignment robustness can be a donor candidate, but donor selection also depends on whether reducing its rank is less harmful than reducing the rank of other modules. To identify safe donors, we conduct a controlled module-wise projection-rank sensitivity ablation on Llama~2~\citep{touvron2023llama2openfoundation} 350M: we reduce one target module's projection rank from $r_{\mathrm{base}}$ to $r_{\mathrm{base}}/2$ at a time while keeping all others at the uniform baseline rank. In our experiments at Llama~2 350M, this corresponds to reducing one target module to $r=128$ while keeping the others at $r=256$ (six configurations in total, one per non-receiver module).

Figure~\ref{app_fig:rank_sensitivity} reports the per-module final perplexity under this ablation. The donor/receiver implications and the underlying architectural reasoning that supports $\mathcal{D} = \{\attn{q}, \attn{k}\}$ are discussed in \S\ref{sec:rank_allocation}.

% ============================================================================
% ============================================================================
% ⬆️⬆️⬆️⬆️ APPENDIX C ⬆️⬆️⬆️⬆️
% ============================================================================
% ============================================================================

% ============================================================================
% ============================================================================
% ⬇️⬇️⬇️⬇️ APPENDIX D ⬇️⬇️⬇️⬇️
% ============================================================================
% ============================================================================

\section{Experimental Setup}
\label{app_sec:exp_setup_detail}

\subsection{Model Architecture}
We follow the Llama configurations used in GaLore \citep{zhao2024galore} across model scales. For Llama~3.2 \citep{grattafiori2024llama3herdmodels}, Qwen2.5 \citep{qwen2025qwen25technicalreport}, Qwen3 \citep{yang2025qwen3technicalreport}, and DeepSeek-V2 \citep{deepseekai2024deepseekv2strongeconomicalefficient}, we derive compatible 350M-scale configurations from their official implementations, starting from our Llama~2 \citep{touvron2023llama2openfoundation} 350M setup. Specifically for DeepSeek-V2, which employs a distinct Multi-head Latent Attention (MLA) architecture, we carefully scale the latent projection dimensions to align with the 350M parameter budget (\texttt{q\_lora\_rank}=768, \texttt{kv\_lora\_rank}=256, \texttt{qk\_nope\_head\_dim}=64, \texttt{qk\_rope\_head\_dim}=32, and \texttt{v\_head\_dim}=64). Table~\ref{app_tab:model_hparams} summarizes the key architectural and training hyperparameters across all evaluated models.

\subsection{Training Parameters}
\label{app_subsec:training_params}
\label{appx:training_params}
\label{appx:block_size_config}
We use the C4 English dataset with a T5 tokenizer \citep{raffel2023exploringlimitstransferlearning}. We use AdamW \citep{loshchilov2019adamw} with \(\beta_1=0.9\), \(\beta_2=0.999\), \(\epsilon=10^{-6}\), zero weight decay, and bias correction enabled. The peak learning rate is \(0.01\) (\(0.005\) for 7B), with 10\% linear warmup and cosine decay to a minimum ratio of 0.1. 
Unless otherwise noted, we use a global batch size of 512 and a maximum sequence length of 256.

For GaLore \citep{zhao2024galore}, we use a scale of 0.25 and an interval \(T=200\). The magnitude branch \(M\) and its optimizer states are maintained in fp32, and the recomposed update is cast to bf16 before application. All runs use \texttt{torch.compile} in default mode. Gradient clipping of 1.0 is applied only to the 7B model.

\paragraph{Block size and rank budget per model scale.}
The block size $B$ used in the main experiments follows the head-dimension-informed default discussed in \S\ref{sec:block_decomposition}:
we choose $B$ in the neighborhood of $d_{\mathrm{head}}$, using $B=32$ for Llama~2 350M ($d_{\mathrm{head}}=64$), $B=64$ for Llama~2 1B ($d_{\mathrm{head}}=64$), and $B=128$ for Llama~2 7B ($d_{\mathrm{head}}=128$).
Per-scale baseline projection ranks $r_{\mathrm{base}}$ follow the GaLore convention: $r_{\mathrm{base}} = 256$ at 350M, $r_{\mathrm{base}} = 512$ at 1B, and $r_{\mathrm{base}} = 1024$ at 7B; module-wise allocation $\{r_\tau\}$ follows the static reallocation rule of Algorithm~\ref{alg:moara_full} with $\mathcal{D} = \{Q, K\}$, $\mathcal{R} = \{\mathrm{MLP.down}\}$, and $\Delta r = r_{\mathrm{base}}/2$ (Appendix~\ref{appx:final_rank_allocation}).

\subsection{Hardware and Profiling Details}
For all step-to-target and time-to-target experiments, GaLore~\citep{zhao2024galore} and MoARa are compared under identical hardware conditions, using NVIDIA RTX A6000, H100, and RTX PRO 6000 GPUs depending on model scale.

For memory profiling, we measure memory usage after the first 50 training steps under the same sequence-length and batch-size settings as in Appendix~\ref{app_subsec:training_params}.

% ============================================================================
% ============================================================================
% ⬆️⬆️⬆️⬆️ APPENDIX D ⬆️⬆️⬆️⬆️
% ============================================================================
% ============================================================================

% ============================================================================
% ============================================================================
% ⬇️⬇️⬇️⬇️ FIGURE: LLAMA 3.2 & QWENs & DEEPSEEK ⬇️⬇️⬇️⬇️
% ============================================================================
% ============================================================================

\begin{figure*}[t]
    \centering
    \includegraphics[width=1\textwidth]{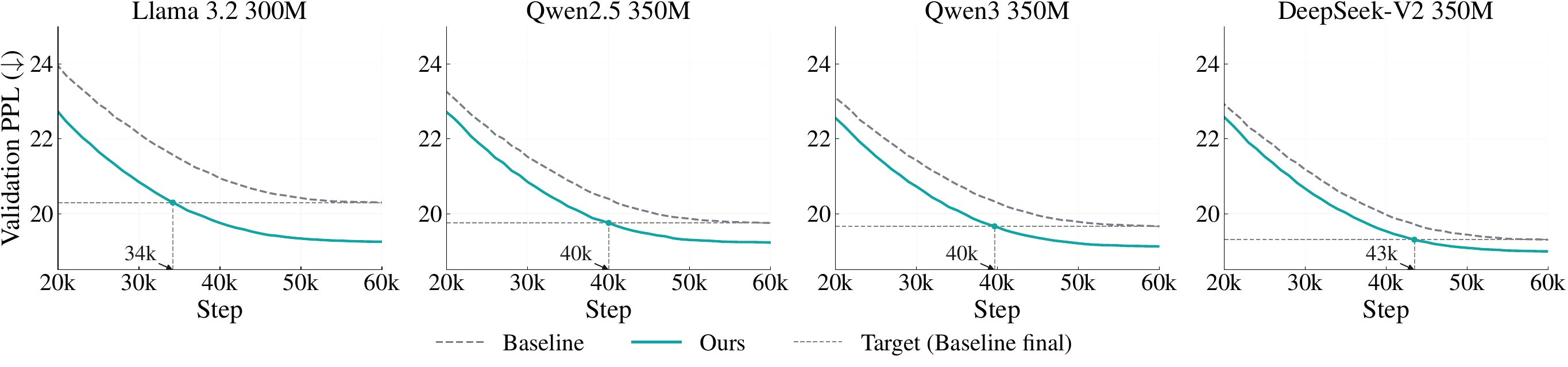}
    \caption{\textbf{Validation perplexity curves across recent Transformer architectures at the 350M scale.} 
    We compare our method with GaLore on Llama 3.2, Qwen2.5, Qwen3, and DeepSeek-V2. 
    The dashed horizontal line marks the final validation perplexity of the baseline. 
    In all cases, our method reaches this target earlier than the baseline.}
    \label{app_fig:step-to-target(qwen/deepseek)}
\end{figure*}

% ============================================================================
% ============================================================================
% ⬆️⬆️⬆️⬆️ FIGURE: LLAMA 3.2 & QWENs & DEEPSEEK ⬆️⬆️⬆️⬆️
% ============================================================================
% ============================================================================

% ============================================================================
% ============================================================================
% ⬇️⬇️⬇️⬇️ APPENDIX E ⬇️⬇️⬇️⬇️
% ============================================================================
% ============================================================================

\section{Additional Experimental Validation}
\label{app_sec:additional_experimental_validation}
\subsection{Random Rank Allocation}
\label{app_subsec:projection_rank_random_distribution}
\label{appx:randomized_allocation}
We performed a validation experiment to verify that our projection-rank redistribution outperforms random budget-preserving allocations.
Under identical projection rank budget constraints, we pre-trained 20 Llama 2 350M models with randomly assigned ranks for 10k steps and compared them with our method. 
The projection ranks and resulting evaluation PPL for each of the 20 runs are summarized in Table~\ref{app_tab:randomrank_table}. 
Our module-aware allocation outperforms all 20 random variants, and none approach its convergence trajectory.

\subsection{Cross-Architecture Generalization}
\label{appx:architecture_results}

Figure~\ref{app_fig:step-to-target(qwen/deepseek)} reports the full validation PPL curves for the cross-architecture study referenced in \S\ref{sec:main_result}. Across Llama~3.2~\citep{grattafiori2024llama3herdmodels}, Qwen2.5~\citep{qwen2025qwen25technicalreport}, Qwen3~\citep{yang2025qwen3technicalreport}, and DeepSeek-V2~\citep{deepseekai2024deepseekv2strongeconomicalefficient}, MoARa reaches the standard GaLore~\citep{zhao2024galore} target earlier in all cases despite the architectural differences (GQA in Llama~3.2 and the Qwen family, MLA in DeepSeek-V2). The cross-scale speedup figure (Figure~\ref{fig:speedup_across_scale}) is included alongside for visual reference.

% ============================================================================
% ============================================================================
% ⬇️⬇️⬇️⬇️ FIGURE: SPEED_UP ⬇️⬇️⬇️⬇️
% ============================================================================
% ============================================================================

\begin{figure}[t]
    \centering
    \includegraphics[width=1\linewidth]{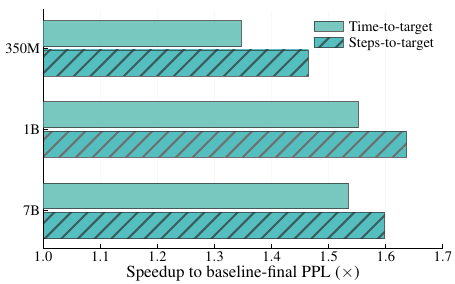}
    \caption{\textbf{Speedup across model scales.} 
    MoARa's speedup over standard GaLore in reaching the baseline's final perplexity, measured by both steps-to-target and time-to-target. GaLore with MoARa achieves up to $1.64\times$ step speedup across all three configurations.}
    \label{fig:speedup_across_scale}
\end{figure}

% ============================================================================
% ============================================================================
% ⬆️⬆️⬆️⬆️ FIGURE: SPEED_UP ⬆️⬆️⬆️⬆️
% ============================================================================
% ============================================================================

% ============================================================================
% ============================================================================
% ⬇️⬇️⬇️⬇️ TABLE: RANDOM 20 ⬇️⬇️⬇️⬇️
% ============================================================================
% ============================================================================

\begin{table}[t]
    \centering
    \caption{\textbf{Random rank allocations and final evaluation perplexity.}
    Each run preserves the same total projection-rank budget. 
    G, U, and D denote the gate, up, and down MLP projections. Lower PPL is better.}
    \label{app_tab:randomrank_table}
    \small
    \setlength{\tabcolsep}{1.4pt}
    \renewcommand{\arraystretch}{1.02}
    \begin{tabular*}{\columnwidth}{@{\extracolsep{\fill}}lrrrrrrrr@{}}
    \toprule
    \textbf{Run} & \textbf{Q} & \textbf{K} & \textbf{V} & \textbf{O}
    & \textbf{G} & \textbf{U} & \textbf{D} & \textbf{PPL} \\
    \midrule
    R01 & 128 & 240 & 560 & 184 & 256 &  80 & 344 & 24.72 \\
    R02 & 184 & 320 & 280 & 320 & 352 & 272 &  64 & 24.85 \\
    R03 & 216 & 208 & 248 & 328 & 376 & 304 & 112 & 24.53 \\
    R04 & 160 & 472 &  72 & 360 & 256 &  80 & 392 & 25.00 \\
    R05 & 352 & 192 &  88 & 448 & 408 & 128 & 176 & 24.88 \\
    R06 &  96 & 136 & 112 & 608 & 456 & 232 & 152 & 24.53 \\
    R07 & 352 & 256 & 192 & 208 & 120 & 536 & 128 & 24.83 \\
    R08 & 552 & 104 & 192 & 128 &  64 & 440 & 312 & 24.92 \\
    R09 & 200 & 360 & 104 & 344 & 376 &  80 & 328 & 24.73 \\
    R10 & 128 & 312 & 416 & 200 & 328 &  72 & 336 & 24.60 \\
    R11 &  72 & 568 & 376 &  80 &  72 & 152 & 472 & 25.29 \\
    R12 & 160 & 208 &  64 & 680 & 128 & 360 & 192 & 24.78 \\
    R13 & 144 & 208 & 344 & 248 & 328 & 312 & 208 & 24.33 \\
    R14 & 256 & 480 &  80 & 232 & 472 & 120 & 152 & 25.17 \\
    R15 & 184 & 224 & 256 & 312 & 560 & 136 & 120 & 24.63 \\
    R16 & 144 & 392 & 360 & 192 & 248 & 336 & 120 & 24.72 \\
    R17 & 232 & 304 &  64 & 424 & 376 & 304 &  88 & 24.88 \\
    R18 & 184 &  64 & 112 & 592 & 176 & 504 & 160 & 24.56 \\
    R19 & 136 &  72 & 720 &  80 & 224 & 288 & 272 & 24.73 \\
    R20 & 152 &  64 & 256 & 536 & 496 &  80 & 208 & 24.58 \\
    \midrule
    \textbf{Ours}
        & \textbf{128} & \textbf{128} & \textbf{256} & \textbf{256}
        & \textbf{256} & \textbf{256} & \textbf{512}
        & \textbf{24.07} \\
    \bottomrule
    \end{tabular*}
\end{table}

% ============================================================================
% ============================================================================
% ⬆️⬆️⬆️⬆️ TABLE: RANDOM 20 ⬆️⬆️⬆️⬆️
% ============================================================================
% ============================================================================

\subsection{Robustness at Sequence Length 2048}
\label{appx:context_2048}
The main scaling and architecture results in \S\ref{sec:main_result} use the GaLore~\citep{zhao2024galore}-standard sequence length of 256, which matches the GaLore baseline's setup and isolates the optimizer-side comparison from sequence-length-dependent variables. 
To test whether the efficiency gains persist beyond this short-context setting, we re-evaluate GaLore with MoARa against standard GaLore at sequence length 2048 on Llama~2~\citep{touvron2023llama2openfoundation} 350M. 
All other training hyperparameters are kept identical to the sequence-length-256 setup of Appendix~\ref{app_subsec:training_params}.

As shown in Figure~\ref{app_fig:context_2048}, GaLore with MoARa reaches standard GaLore's final validation perplexity in $35.0\%$ fewer steps and $33.3\%$ less wall-clock time, and improves final perplexity by $2.65\%$ (PPL $18.09 \to 17.61$). 
The reduction profile in both steps and wall-clock time is consistent with the sequence-length-256 result at the same model scale.
This single-scale experiment supports an optimizer-side interpretation but does not establish sequence-length invariance at larger scales or substantially longer contexts.

% ============================================================================
% ============================================================================
% ⬇️⬇️⬇️⬇️ FIGURE: CONTEXT_2048 ⬇️⬇️⬇️⬇️
% ============================================================================
% ============================================================================

\begin{figure}[t]
    \centering
    \includegraphics[width=1\linewidth]{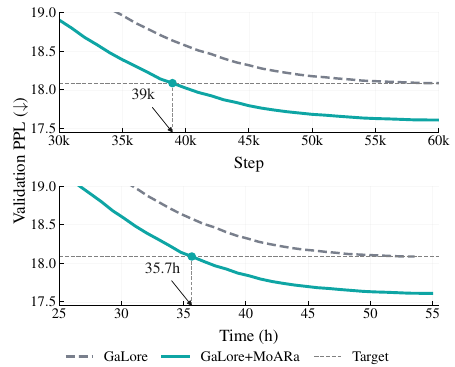}
    \caption{\textbf{Sequence-length 2048 validation perplexity on Llama 2 350M.} 
    Validation perplexity versus training steps (top) and wall-clock time (bottom); 
    annotated points mark where GaLore with MoARa reaches the standard GaLore target (dashed line).}
    \label{app_fig:context_2048}
\end{figure}

% ============================================================================
% ============================================================================
% ⬆️⬆️⬆️⬆️ FIGURE: CONTEXT_2048 ⬆️⬆️⬆️⬆️
% ============================================================================
% ============================================================================

\subsection{Multi-seed Robustness}
\label{app_subsec:1B_multiseed_analysis}
Figure~\ref{app_fig:1B_multiseed_ppl_curve} shows the full 1B pre-training trajectories over three random seeds. Our method consistently outperforms the baseline in all runs, which supports the robustness claim in \S\ref{sec:main_result}.

\subsection{Q-GaLore Precision Analysis}
\label{appx:qgalore_precision}
We isolate whether the magnitude-branch optimizer-state precision explains the Group C behavior observed in \S\ref{sec:analysis}.
The official Q-GaLore configuration uses INT4 projection matrices, INT8 weights, and an 8-bit Adam optimizer.
We therefore evaluate Q-GaLore + decomposition while changing only the magnitude-branch optimizer state from 8-bit to FP32; the INT4 projection matrices and INT8 weights remain unchanged.
This change does not recover the benefit of decomposition.
Thus, the magnitude-branch optimizer-state precision is unlikely to be the primary bottleneck.
Because projection and weight quantization remain active together in this ablation, we do not attribute the limitation to either component individually.
Rather, the result supports a broader quantization-compatibility limitation for decomposition.
A direct comparison with GaLore + decomposition is consistent with this interpretation: decomposition alone yields a final perplexity comparable to vanilla GaLore, whereas the same decomposition under Q-GaLore produces a small but consistent increase in final perplexity.

% ============================================================================
% ============================================================================
% ⬇️⬇️⬇️⬇️ FIGURE: 1B_MULTISEED ⬇️⬇️⬇️⬇️
% ============================================================================
% ============================================================================

\begin{figure}
    \centering
    \includegraphics[width=1\linewidth]{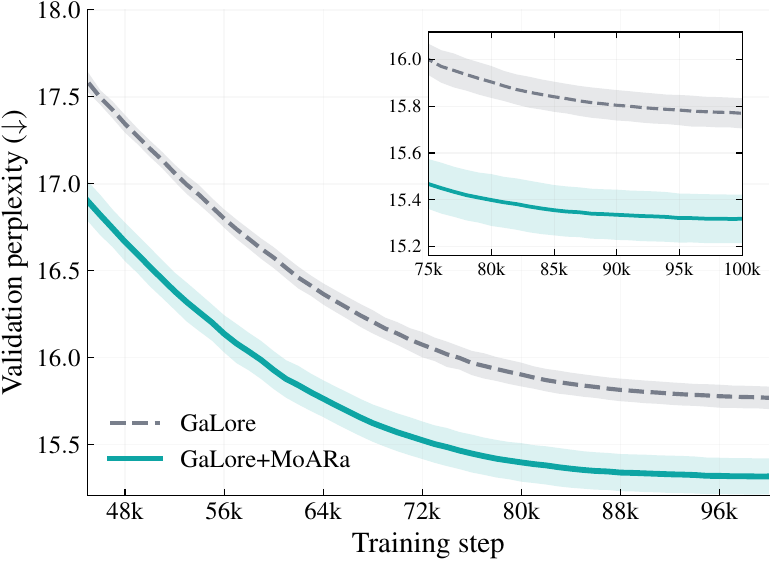}
    \caption{\textbf{Validation perplexity across three random seeds on Llama 2 1B.} Our method consistently outperforms the baseline. The inset zooms into the final phase.}
    \label{app_fig:1B_multiseed_ppl_curve}
\end{figure}

% ============================================================================
% ============================================================================
% ⬆️⬆️⬆️⬆️ FIGURE: 1B_MULTISEED ⬆️⬆️⬆️⬆️
% ============================================================================
% ============================================================================

% ============================================================================
% ============================================================================
% ⬆️⬆️⬆️⬆️ APPENDIX E ⬆️⬆️⬆️⬆️
% ============================================================================
% ============================================================================

%\clearpage

% ============================================================================
% ============================================================================
% ⬇️⬇️⬇️⬇️ APPENDIX F ⬇️⬇️⬇️⬇️
% ============================================================================
% ============================================================================

\section{Detailed Benchmark Results}
\label{app_sec:benchmarks}
\label{appx:downstream}

This section reports the full benchmark tables that support the summary results in \S\ref{sec:main_result}. Table~\ref{app_tab:1B_benchmarks_full} shows the three-seed evaluation at the final checkpoint on seven downstream benchmarks (HellaSwag~\citep{zellers2019hellaswag}, WinoGrande~\citep{sakaguchi2020winogrande}, PIQA~\citep{bisk2020piqa}, BoolQ~\citep{clark2019boolq}, ARC-Challenge and ARC-Easy~\citep{clark2018think}, and TruthfulQA~\citep{lin2022truthfulqa}) together with C4~\citep{raffel2023exploringlimitstransferlearning} and WikiText-2~\citep{merity2017pointer} perplexity, reporting mean and standard deviation over three random seeds. Table~\ref{app_tab:60k_vs_100k} compares our 60k-step intermediate checkpoint with the 100k-step GaLore~\citep{zhao2024galore} baseline to demonstrate token efficiency rather than final-checkpoint parity.

\paragraph{Multi-seed robustness at the final checkpoint.}
The three-seed evaluation in Table~\ref{app_tab:1B_benchmarks_full} shows that GaLore~\citep{zhao2024galore} with MoARa is comparable to or outperforms GaLore on the vast majority of metrics.
MoARa has lower standard deviation on all six perplexity metrics and on several downstream metrics.
The consistent mean trends across seeds indicate that the reported gains are not driven by a single favorable seed.

\paragraph{ARC-Challenge under short-context pre-training.}
Among the summary metrics reported in Table~\ref{tab:downstream_1b}, ARC-Challenge~\citep{clark2018think} normalized accuracy is the only metric that decreases.
We interpret this gap with caution.
ARC-Challenge is a four-way multiple-choice benchmark in which chance-level accuracy is $0.25$, and both models attain mean normalized accuracy within $0.003$ of chance level. 
The task which requires multi-step science reasoning and remains challenging for our 1B model pre-trained at sequence length 256 sits very near the chance-level floor for both checkpoints, so the per-seed differences reflect noise around that floor rather than a substantive capability gap. 
This interpretation is consistent with the modest seed-to-seed variation reported in Table~\ref{app_tab:1B_benchmarks_full}.

% ============================================================================
% ============================================================================
% ⬇️⬇️⬇️⬇️ TABLE: 1B_FULL_BENCHMARKS ⬇️⬇️⬇️⬇️
% ============================================================================
% ============================================================================

\begin{table*}[t]
    \centering
    \caption{\textbf{Detailed zero-shot benchmark results on Llama 2 1B.}
    Both standard GaLore and GaLore with MoARa are evaluated at their final 100k-step checkpoints. We report per-seed scores, mean $\mu$, and standard deviation $\sigma$ over three seeds. We abbreviate bits-per-byte, byte-perplexity, word-perplexity, and normalized accuracy as bpb, bppl, wppl, and acc\_n.}
    \label{app_tab:1B_benchmarks_full}
    \small
    \setlength{\tabcolsep}{3.2pt}
    \renewcommand{\arraystretch}{1.06}
    \begin{tabular*}{\textwidth}{@{\extracolsep{\fill}}llc*{5}{r}*{5}{r}@{}}
    \toprule
    \textbf{Task} & \textbf{Metric} & \textbf{Dir.}
    & \multicolumn{5}{c}{\textbf{GaLore}}
    & \multicolumn{5}{c}{\textbf{GaLore with MoARa}} \\
    \cmidrule(lr){4-8}
    \cmidrule(l){9-13}
    & & 
    & \textbf{1} & \textbf{2} & \textbf{3} & $\boldsymbol{\mu}$ & $\boldsymbol{\sigma}$
    & \textbf{1} & \textbf{2} & \textbf{3} & $\boldsymbol{\mu}$ & $\boldsymbol{\sigma}$ \\
    \midrule

    \multirow{3}{*}{C4}
    & bpb  & $\downarrow$ & 1.077 & 1.091 & 1.060 & 1.076 & 0.016 & 1.056 & 1.063 & 1.073 & \textbf{1.064} & 0.008 \\
    & bppl & $\downarrow$ & 2.110 & 2.130 & 2.084 & 2.108 & 0.023 & 2.079 & 2.089 & 2.103 & \textbf{2.091} & 0.012 \\
    & wppl & $\downarrow$ & 87.18 & 92.34 & 81.13 & 86.88 & 5.612 & 79.96 & 82.29 & 85.61 & \textbf{82.62} & 2.837 \\
    \addlinespace[1pt]
    
    \multirow{3}{*}{WikiText}
    & bpb  & $\downarrow$ & 1.135 & 1.136 & 1.103 & 1.125 & 0.019 & 1.108 & 1.098 & 1.111 & \textbf{1.106} & 0.007 \\
    & bppl & $\downarrow$ & 2.197 & 2.197 & 2.148 & 2.181 & 0.029 & 2.156 & 2.141 & 2.159 & \textbf{2.152} & 0.010 \\
    & wppl & $\downarrow$ & 67.218 & 67.332 & 59.588 & 64.712 & 4.439 & 60.856 & 58.616 & 61.334 & \textbf{60.269} & 1.451 \\
    
    \multirow{2}{*}{ARC-C}
    & acc    & $\uparrow$ & 0.212 & 0.215 & 0.212 & \textbf{0.213} & 0.002 & 0.206 & 0.202 & 0.197 & 0.202 & 0.004 \\
    & acc\_n & $\uparrow$ & 0.255 & 0.248 & 0.252 & \textbf{0.252} & 0.004 & 0.244 & 0.257 & 0.244 & 0.248 & 0.007 \\
    \addlinespace[1pt]
    
    \multirow{2}{*}{ARC-E}
    & acc    & $\uparrow$ & 0.477 & 0.469 & 0.484 & 0.477 & 0.008 & 0.487 & 0.476 & 0.481 & \textbf{0.482} & 0.006 \\
    & acc\_n & $\uparrow$ & 0.423 & 0.421 & 0.430 & 0.425 & 0.005 & 0.436 & 0.434 & 0.430 & \textbf{0.433} & 0.003 \\
    \addlinespace[1pt]
    
    BoolQ
    & acc & $\uparrow$ & 0.586 & 0.556 & 0.495 & 0.546 & 0.046 & 0.565 & 0.557 & 0.552 & \textbf{0.558} & 0.007 \\
    \addlinespace[1pt]
    
    \multirow{2}{*}{HellaSwag}
    & acc    & $\uparrow$ & 0.326 & 0.325 & 0.326 & 0.326 & 0.001 & 0.329 & 0.331 & 0.332 & \textbf{0.331} & 0.001 \\
    & acc\_n & $\uparrow$ & 0.386 & 0.389 & 0.384 & 0.386 & 0.002 & 0.402 & 0.394 & 0.399 & \textbf{0.398} & 0.004 \\
    \addlinespace[1pt]
    
    \multirow{2}{*}{PIQA}
    & acc    & $\uparrow$ & 0.679 & 0.681 & 0.680 & \textbf{0.680} & 0.001 & 0.675 & 0.677 & 0.682 & 0.678 & 0.004 \\
    & acc\_n & $\uparrow$ & 0.674 & 0.677 & 0.678 & 0.677 & 0.002 & 0.674 & 0.684 & 0.687 & \textbf{0.682} & 0.007 \\
    \addlinespace[1pt]
    
    WinoGrande
    & acc & $\uparrow$ & 0.504 & 0.526 & 0.508 & 0.513 & 0.012 & 0.527 & 0.539 & 0.526 & \textbf{0.531} & 0.007 \\
    \addlinespace[1pt]
    
    TruthfulQA
    & mc2 & $\uparrow$ & 0.398 & 0.421 & 0.407 & \textbf{0.409} & 0.011 & 0.417 & 0.393 & 0.408 & 0.406 & 0.012 \\
    \addlinespace[1pt]
        
    \bottomrule
    \end{tabular*}
\end{table*}

% ============================================================================
% ============================================================================
% ⬆️⬆️⬆️⬆️ TABLE: 1B_FULL_BENCHMARKS ⬆️⬆️⬆️⬆️
% ============================================================================
% ============================================================================

% ============================================================================
% ============================================================================
% ⬇️⬇️⬇️⬇️ TABLE: 1B_BENCHMARKS_60K_vs_100K ⬇️⬇️⬇️⬇️
% ============================================================================
% ============================================================================

\begin{table*}[t]
    \centering
    \caption{\textbf{Zero-shot benchmarks on Llama 2 1B: 60k vs. 100k.}
    We compare standard GaLore at 100k steps with GaLore with MoARa at 60k steps. $\Delta$ denotes the difference between GaLore with MoARa and GaLore. $\pm$ denotes the standard error reported by the evaluation harness. Results are from a single seed. We abbreviate bits-per-byte, byte-perplexity, word-perplexity, and normalized accuracy as bpb, bppl, wppl, and acc\_n.}
    \label{app_tab:60k_vs_100k}
    \small
    \setlength{\tabcolsep}{5.5pt}
    \renewcommand{\arraystretch}{1.08}
    
    \begin{tabular*}{\textwidth}{@{\extracolsep{\fill}}llcrrr@{}}
    \toprule
    \textbf{Task} & \textbf{Metric} & \textbf{Dir.}
    & \textbf{GaLore}
    & \textbf{GaLore+MoARa}
    & $\boldsymbol{\Delta}$ \\
    \midrule

    \multirow{3}{*}{C4}
    & bpb  & $\downarrow$ & $1.060$ & $1.085$ & $+0.025$ \\
    & bppl & $\downarrow$ & $2.084$ & $2.121$ & $+0.036$ \\
    & wppl & $\downarrow$ & $81.130$ & $89.981$ & $+8.851$ \\
    \addlinespace[1pt]
    
    \multirow{3}{*}{WikiText}
    & bpb  & $\downarrow$ & $1.103$ & $1.125$ & $+0.022$ \\
    & bppl & $\downarrow$ & $2.148$ & $2.180$ & $+0.033$ \\
    & wppl & $\downarrow$ & $59.588$ & $64.603$ & $+5.016$ \\
    
    \multirow{2}{*}{ARC-C}
    & acc    & $\uparrow$ & $0.212 \pm 0.012$ & $0.192 \pm 0.012$ & $-0.020$ \\
    & acc\_n & $\uparrow$ & $0.252 \pm 0.013$ & $0.247 \pm 0.013$ & $-0.005$ \\
    \addlinespace[1pt]
    
    \multirow{2}{*}{ARC-E}
    & acc    & $\uparrow$ & $0.484 \pm 0.010$ & $0.475 \pm 0.010$ & $-0.009$ \\
    & acc\_n & $\uparrow$ & $0.430 \pm 0.010$ & $0.422 \pm 0.010$ & $-0.008$ \\
    \addlinespace[1pt]
    
    BoolQ
    & acc & $\uparrow$ & $0.495 \pm 0.009$ & $0.561 \pm 0.009$ & $+0.066$ \\
    \addlinespace[1pt]
    
    \multirow{2}{*}{HellaSwag}
    & acc    & $\uparrow$ & $0.326 \pm 0.005$ & $0.329 \pm 0.005$ & $+0.003$ \\
    & acc\_n & $\uparrow$ & $0.384 \pm 0.005$ & $0.393 \pm 0.005$ & $+0.008$ \\
    \addlinespace[1pt]
    
    \multirow{2}{*}{PIQA}
    & acc    & $\uparrow$ & $0.680 \pm 0.011$ & $0.676 \pm 0.011$ & $-0.004$ \\
    & acc\_n & $\uparrow$ & $0.678 \pm 0.011$ & $0.675 \pm 0.011$ & $-0.003$ \\
    \addlinespace[1pt]
    
    WinoGrande
    & acc & $\uparrow$ & $0.508 \pm 0.014$ & $0.511 \pm 0.014$ & $+0.004$ \\
    \addlinespace[1pt]
    
    TruthfulQA
    & mc2 & $\uparrow$ & $0.407 \pm 0.015$ & $0.408 \pm 0.015$ & $+0.001$ \\
    \addlinespace[1pt]
    
    \bottomrule
    \end{tabular*}
\end{table*}

% ============================================================================
% ============================================================================
% ⬆️⬆️⬆️⬆️ TABLE: 1B_BENCHMARKS_60K_vs_100K ⬆️⬆️⬆️⬆️
% ============================================================================
% ============================================================================

\paragraph{60k MoARa versus 100k GaLore: downstream behavior at an earlier checkpoint.}
Table~\ref{app_tab:60k_vs_100k} compares the 60k-step intermediate checkpoint from our 100k-step GaLore~\citep{zhao2024galore} with MoARa run against the final 100k-step GaLore checkpoint.
This checkpoint is selected because its pretraining validation perplexity is close to GaLore's final training-run value.
On the separate C4 and WikiText evaluation metrics in Table~\ref{app_tab:60k_vs_100k}, however, GaLore retains a clear perplexity advantage, as MoARa has seen approximately 40\% fewer training tokens ($7.86\text{B}$ vs.\ $13.1\text{B}$ at a context length of $256$).

Despite this perplexity gap, the downstream picture is different. Ten of the eleven reported downstream comparisons fall within two standard errors of the GaLore baseline, and on BoolQ the 60k GaLore with MoARa checkpoint actually exceeds the 100k GaLore baseline. 
The reductions in downstream accuracy that do appear are uniformly small in absolute terms.
ARC-Challenge remains near the chance-level floor in normalized accuracy, whereas the ARC-Easy and PIQA differences are smaller than the corresponding standard errors reported by the evaluation harness.

\paragraph{Why does the downstream gap not track the perplexity gap?}
Two complementary explanations are consistent with this pattern. 
First, one possible explanation is that the additional 40\% of training tokens primarily refine the next-token distribution along directions that lower per-token surprisal on general-domain text without unlocking qualitatively new downstream behaviors. 
This is consistent with the broader observation that perplexity continues to decrease smoothly with more tokens, while downstream capability often follows a more saturating curve in this token range.
Second, our main step-to-target results (\S\ref{sec:main_result}) already demonstrate that GaLore with MoARa reaches GaLore's perplexity target with substantially fewer tokens; Table~\ref{app_tab:60k_vs_100k} extends this token-efficiency observation by showing that downstream behavior is largely retained at the point where GaLore with MoARa is still trailing GaLore in raw perplexity. 
We do not interpret the BoolQ improvement at 60k as evidence of a separate generalization gain.

% ============================================================================
% ============================================================================
% ⬆️⬆️⬆️⬆️ APPENDIX F ⬆️⬆️⬆️⬆️
% ============================================================================
% ============================================================================

% ============================================================================
% ============================================================================
% ⬇️⬇️⬇️⬇️ APPENDIX G ⬇️⬇️⬇️⬇️
% ============================================================================
% ============================================================================

\section{Detailed Memory Analysis}
\label{appendix:vram_analysis}
\label{appx:memory}

This section provides the comprehensive memory measurements used in \S\ref{sec:main_result}. 

For a target matrix $W$, let $r_W$ denote its projection rank and $d_W$ the state dimension not compressed by the projection. Let $b_V$ and $b_M$ denote the bytes per entry in each direction-branch and magnitude-branch moment tensor, respectively. 
The first- and second-moment states of the projected direction require approximately
\[\mathcal{M}_{\mathrm{dir}}(W)=2 b_V r_W d_W.\]

Relative to the uniform baseline rank $r_{\mathrm{base}}$, projection rank reallocation changes the persistent optimizer-state cost by
\[\Delta\mathcal{M}_{\mathrm{rank}}=2 b_V\sum_W\left(r_W-r_{\mathrm{base}}\right)d_W.\]

For $W\in\mathbb{R}^{m_W\times n_W}$, the block-magnitude tensor has shape
\[m_W\times\left\lceil\frac{n_W}{B_W}\right\rceil,\]
so its Adam states add approximately
\[\Delta\mathcal{M}_{\mathrm{mag}}=2 b_M\sum_Wm_W\left\lceil\frac{n_W}{B_W}\right\rceil.\]

The sums include all target matrices across all layers, and the factor of $2$ accounts for Adam's first and second moments. 
Because donor and receiver matrices can have different $d_W$, preserving the total projection-rank budget does not imply identical optimizer-state bytes. 
At 7B, the measured optimizer-state increases are 864~MiB from projection-rank allocation and 386~MiB from block-wise decomposition; together they match the 1,250~MiB increase of full MoARa.

Table~\ref{app_tab:memory_profiling_full}(a) reports scale-wise memory profiling under both eager and compiled settings; notably, under the compiled setting, most intermediate overhead is absorbed, leaving the remaining cost dominated by the added optimizer states. 
Furthermore, Table~\ref{app_tab:memory_profiling_full}(b) breaks down the memory cost of each component, with all measurements taken in eager mode to effectively isolate the algorithmic overhead. 
This granular breakdown allows us to separate the specific memory requirements of block-wise decomposition from those of projection rank allocation, showing that the two components combine approximately additively across model scales.

% ============================================================================
% ============================================================================
% ⬆️⬆️⬆️⬆️ APPENDIX G ⬆️⬆️⬆️⬆️
% ============================================================================
% ============================================================================

% ============================================================================
% ============================================================================
% ⬇️⬇️⬇️⬇️ APPENDIX H ⬇️⬇️⬇️⬇️
% ============================================================================
% ============================================================================

\section{Cosine similarity of \texorpdfstring{$G$}{G} and \texorpdfstring{$G_{\text{recon}}$}{G\_recon}}
\label{app_sec:cosine}
\label{appx:cosine_trajectories}

This section reports module-wise cosine similarity trajectories between the full-rank gradient \(G\) and the reconstructed gradient \(G_{\text{recon}}\). These plots support the alignment analysis in \S\ref{sec:rank_allocation} and show how our projection rank reallocation improves low-alignment bottleneck modules across models.

\subsection{Llama 2 350M}
We first show the full module-wise cosine similarity trajectories on Llama~2 \citep{touvron2023llama2openfoundation} 350M in Figure~\ref{app_fig:cosine-350m}. The baseline shows the lowest alignment on \texttt{mlp.down}, while our method improves this bottleneck by reallocating projection rank from robust donor modules.

\subsection{Llama 2 1B}
Figure~\ref{app_fig:cosine-1b} shows the same analysis on Llama~2 \citep{touvron2023llama2openfoundation} 1B. The overall pattern is consistent with the 350M case: the main improvement appears on the most projection-sensitive modules.

\subsection{Llama 2 7B}
Figure~\ref{app_fig:cosine-7b} reports the trajectories on Llama~2 \citep{touvron2023llama2openfoundation} 7B. The receiver-side alignment bottleneck and its improvement remain consistent at the larger scale.

\subsection{Llama 3.2 300M}
Figure~\ref{app_fig:cosine-llama32} shows the trajectories on Llama~3.2 \citep{grattafiori2024llama3herdmodels} 300M.
The same trend appears in this architecture as well, supporting the cross-architecture consistency of our alignment-based analysis.

\subsection{Qwen2.5 350M}
Figure~\ref{app_fig:cosine-qwen25} shows the trajectories on Qwen2.5 \citep{qwen2025qwen25technicalreport} 350M. 
Again, our method improves the lowest-alignment receiver while exhibiting the expected donor-side alignment reductions.

\subsection{Qwen3 350M}
Figure~\ref{app_fig:cosine-qwen3} shows the trajectories on Qwen3 \citep{yang2025qwen3technicalreport} 350M. The result is consistent with the earlier models and supports the same bottleneck-resolution pattern.

\subsection{DeepSeek-V2 350M}
Figure~\ref{app_fig:cosine-deepseekv2} reports the trajectories on the dense variant of DeepSeek-V2 \citep{deepseekai2024deepseekv2strongeconomicalefficient} 350M. Although the attention module structure differs from standard MHA-based models, the same qualitative pattern remains: our method improves the most sensitive modules while keeping the donor-side degradation small.

% ============================================================================
% ============================================================================
% ⬇️⬇️⬇️⬇️ TABLE: MEMORY_PROFILING ⬇️⬇️⬇️⬇️
% ============================================================================
% ============================================================================
\begin{table*}[t]
    \centering
    \caption{\textbf{Memory profiling and component-wise memory ablations.}
    Values are reported in MiB. Panel (a) compares eager (\texttt{OFF}) and
    compiled (\texttt{ON}) settings across model scales; $\Delta$ denotes the
    increase of our method over the baseline. Panel (b) reports component-wise
    memory ablations; $\Delta$ rows show the relative change from the baseline.}
    \label{app_tab:memory_profiling_full}
    \normalsize
    
    \textbf{(a) Memory profiling under eager and compiled settings}
    \vspace{0.25em}
    
    \setlength{\tabcolsep}{5.5pt}
    \renewcommand{\arraystretch}{1.06}
    \begin{tabular*}{\textwidth}{@{\extracolsep{\fill}}llrrrrrr@{}}
    \toprule
    \textbf{Setting} & \textbf{Metric}
    & \multicolumn{2}{c}{\textbf{350M}}
    & \multicolumn{2}{c}{\textbf{1B}}
    & \multicolumn{2}{c}{\textbf{7B}} \\
    \cmidrule(lr){3-4}
    \cmidrule(lr){5-6}
    \cmidrule(l){7-8}
    & & \textbf{Base} & \textbf{Ours}
      & \textbf{Base} & \textbf{Ours}
      & \textbf{Base} & \textbf{Ours} \\
    \midrule
    
    \multirow{4}{*}{\texttt{OFF}}
    & Max alloc.    & 47258.61 & 47316.03 & 51611.42 & 51611.42 & 60898.64 & 64790.63 \\
    & $\Delta$      & \multicolumn{2}{c}{+57.42 (+0.12\%)}
                     & \multicolumn{2}{c}{+0.00 (+0.00\%)}
                     & \multicolumn{2}{c}{+3891.99 (+6.39\%)} \\
    & Max reserved  & 48154.00 & 48210.00 & 54988.00 & 55304.00 & 62122.00 & 66738.00 \\
    & $\Delta$      & \multicolumn{2}{c}{+56.00 (+0.12\%)}
                     & \multicolumn{2}{c}{+316.00 (+0.57\%)}
                     & \multicolumn{2}{c}{+4616.00 (+7.43\%)} \\
    
    \midrule
    
    \multirow{4}{*}{\texttt{ON}}
    & Max alloc.    & 33740.11 & 33804.05 & 38208.39 & 38208.77 & 61777.64 & 61777.68 \\
    & $\Delta$      & \multicolumn{2}{c}{\textbf{+63.94 (+0.19\%)}}
                     & \multicolumn{2}{c}{\textbf{+0.38 (+0.00\%)}}
                     & \multicolumn{2}{c}{\textbf{+0.04 (+0.00\%)}} \\
    & Max reserved  & 34632.00 & 34654.00 & 40850.00 & 41166.00 & 63486.00 & 63614.00 \\
    & $\Delta$      & \multicolumn{2}{c}{\textbf{+22.00 (+0.06\%)}}
                     & \multicolumn{2}{c}{\textbf{+316.00 (+0.77\%)}}
                     & \multicolumn{2}{c}{\textbf{+128.00 (+0.20\%)}} \\
    
    \midrule
    
    \multirow{2}{*}{\texttt{OFF/ON}}
    & Opt. state    & 538.57 & 650.88 & 1652.34 & 1956.57 & 7177.02 & 8427.02 \\
    & $\Delta$      & \multicolumn{2}{c}{+112.31 (+20.85\%)}
                     & \multicolumn{2}{c}{+304.23 (+18.41\%)}
                     & \multicolumn{2}{c}{+1250.00 (+17.42\%)} \\
    
    \bottomrule
    \end{tabular*}
    
    \vspace{0.9em}
    
    \textbf{(b) Component-wise memory ablation}
    \vspace{0.25em}
    
    \setlength{\tabcolsep}{7pt}
    \renewcommand{\arraystretch}{1.06}
    \begin{tabular*}{\textwidth}{@{\extracolsep{\fill}}llrrrr@{}}
    \toprule
    \textbf{Model} & \textbf{Metric}
    & \textbf{Baseline}
    & \textbf{Decomp.}
    & \textbf{Rank Alloc.}
    & \textbf{Ours} \\
    \midrule
    
    \multirow{6}{*}{350M}
    & Max alloc.   & 47258.61 & 47280.14 & 47258.61 & 47316.03 \\
    & $\Delta$     & --       & +0.05\%  & +0.00\%  & +0.12\% \\
    & Max reserved & 48154.00 & 48232.00 & 48130.00 & 48210.00 \\
    & $\Delta$     & --       & +0.16\%  & -0.05\%  & +0.12\% \\
    & Opt. state   & 538.57   & 610.75   & 578.69   & 650.88 \\
    & $\Delta$     & --       & +13.40\% & +7.45\%  & +20.85\% \\
    
    \midrule
    
    \multirow{6}{*}{1B}
    & Max alloc.   & 51611.42 & 51611.42 & 51611.42 & 51611.42 \\
    & $\Delta$     & --       & +0.00\%  & +0.00\%  & +0.00\% \\
    & Max reserved & 54988.00 & 55150.00 & 55132.00 & 55304.00 \\
    & $\Delta$     & --       & +0.29\%  & +0.26\%  & +0.57\% \\
    & Opt. state   & 1652.34  & 1796.58  & 1812.32  & 1956.57 \\
    & $\Delta$     & --       & +8.73\%  & +9.68\%  & +18.41\% \\
    
    \midrule
    
    \multirow{6}{*}{7B}
    & Max alloc.   & 60898.64 & 64790.63 & 64790.63 & 64790.63 \\
    & $\Delta$     & --       & +6.39\%  & +6.39\%  & +6.39\% \\
    & Max reserved & 62122.00 & 66738.00 & 66610.00 & 66738.00 \\
    & $\Delta$     & --       & +7.43\%  & +7.22\%  & +7.43\% \\
    & Opt. state   & 7177.02  & 7563.02  & 8041.02  & 8427.02 \\
    & $\Delta$     & --       & +5.38\%  & +12.04\% & +17.42\% \\
    
    \bottomrule
    \end{tabular*}
\end{table*}

% ============================================================================
% ============================================================================
% ⬆️⬆️⬆️⬆️ TABLE: MEMORY_PROFILING ⬆️⬆️⬆️⬆️
% ============================================================================
% ============================================================================

% ============================================================================
% ============================================================================
% ⬆️⬆️⬆️⬆️ APPENDIX H ⬆️⬆️⬆️⬆️
% ============================================================================
% ============================================================================

% ============================================================================
% ============================================================================
% ⬇️⬇️⬇️⬇️ APPENDIX I ⬇️⬇️⬇️⬇️
% ============================================================================
% ============================================================================

\section{Licenses and Intended Use}

All artifacts used in this work are public research artifacts used in an
academic research context consistent with their intended use.

\paragraph{Dataset and tokenizer.}
The C4 English corpus~\citep{raffel2023exploringlimitstransferlearning} is
released by AllenAI under ODC-BY 1.0, with use also subject to the Common
Crawl terms of use. The T5 tokenizer is released by Google under
Apache-2.0.

\paragraph{Model architectures.}
Llama~2~\citep{touvron2023llama2openfoundation} and Llama~3.2~\citep{grattafiori2024llama3herdmodels}
reference implementations are released by Meta under the Llama~2 Community
License and the Llama~3.2 Community License, respectively. Qwen2.5~\citep{qwen2025qwen25technicalreport}
and Qwen3~\citep{yang2025qwen3technicalreport} reference implementations are
released by Alibaba under Apache-2.0. The DeepSeek-V2~\citep{deepseekai2024deepseekv2strongeconomicalefficient}
reference implementation is released under MIT (code license); the
separately released pretrained weights, which we do not use, are under the
DeepSeek Model License. For all five architectures, we use only the
reference implementations to derive 350M-scale configurations for
from-scratch pre-training and do not redistribute any original pretrained
checkpoints.

\paragraph{Baseline optimizer implementations.}
The reference implementations of GaLore~\citep{zhao2024galore}, Fira~\citep{chen2025fira}, Q-GaLore~\citep{zhang2025qgalore}, OSD~\citep{liang2024osd}, and LDAdam~\citep{robert2025ldadam} are released under Apache-2.0. 
The SubTrack++~\citep{rajabi2025subtrackpp} reference implementation did not include an explicit license file at the time of access.
We use it only as an academic evaluation baseline and do not redistribute its code.

\paragraph{Released artifacts.}
We release no new pretrained model checkpoints or datasets in this
submission.

% ============================================================================
% ============================================================================
% ⬆️⬆️⬆️⬆️ APPENDIX I ⬆️⬆️⬆️⬆️
% ============================================================================
% ============================================================================

\clearpage

\begin{figure*}
  \centering
  \includegraphics[width=\textwidth]{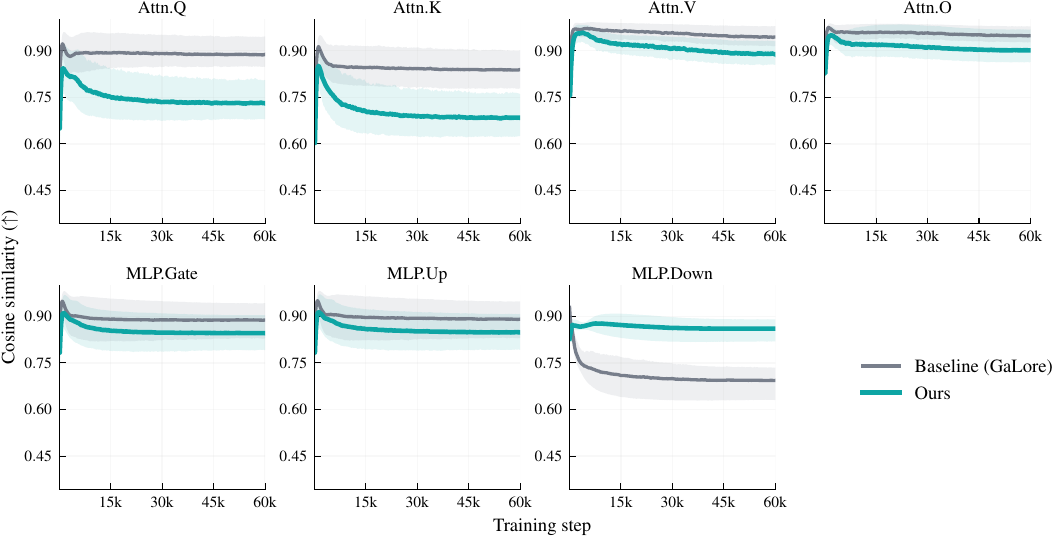}
  \caption{\textbf{Cosine similarity trajectories on Llama 2 350M.} We track the cosine similarity between \(G\) and \(G_{\text{recon}}\) for each target module during pre-training.}
  \label{app_fig:cosine-350m}
\end{figure*}

\begin{figure*}
  \centering
  \includegraphics[width=\textwidth]{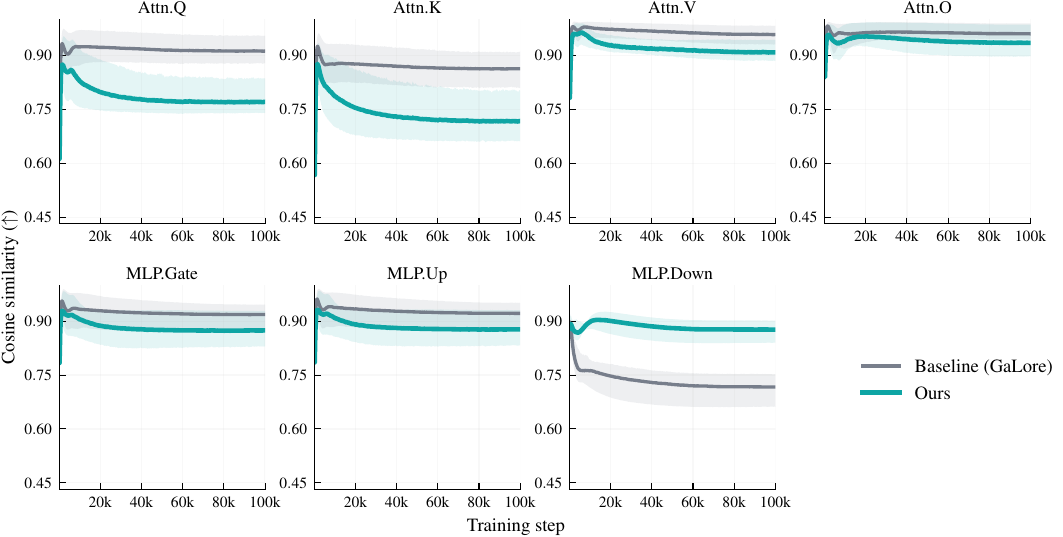}
  \caption{\textbf{Cosine similarity trajectories on Llama 2 1B.} The same module-wise alignment analysis is repeated at the 1B scale.}
  \label{app_fig:cosine-1b}
\end{figure*}

\begin{figure*}
  \centering
  \includegraphics[width=\textwidth]{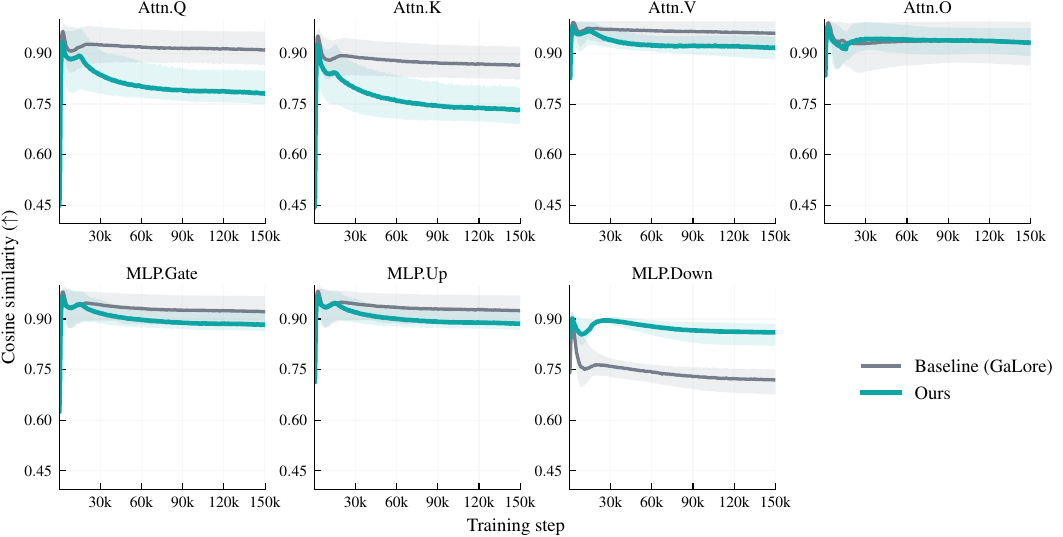}
  \caption{\textbf{Cosine similarity trajectories on Llama 2 7B.} The same module-wise alignment analysis is repeated at the 7B scale.}
  \label{app_fig:cosine-7b}
\end{figure*}

\begin{figure*}
  \centering
  \includegraphics[width=\textwidth]{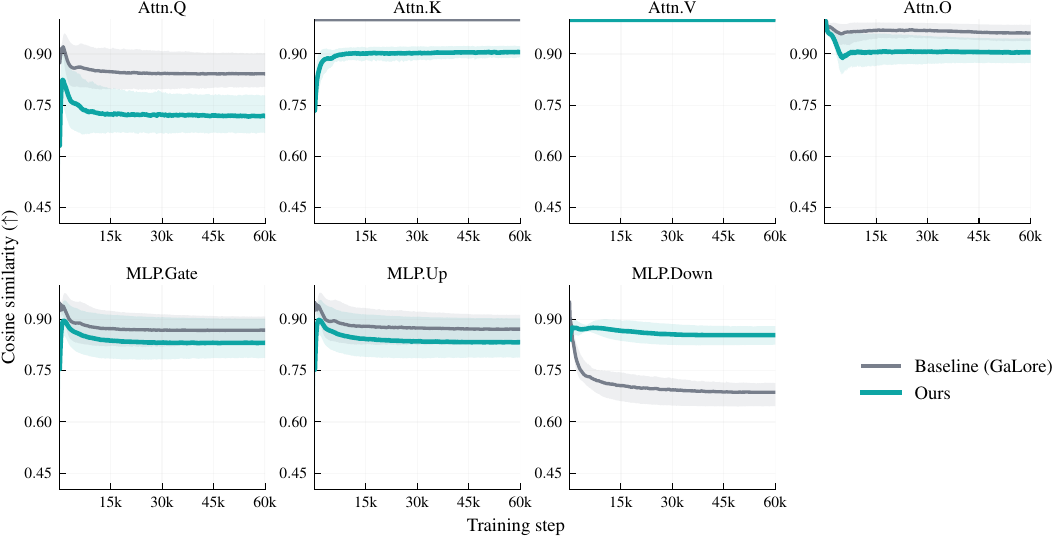}
  \caption{\textbf{Cosine similarity trajectories on Llama 3.2 300M.} The same analysis is applied to Llama 3.2.}
  \label{app_fig:cosine-llama32}
\end{figure*}

\begin{figure*}
  \centering
  \includegraphics[width=\textwidth]{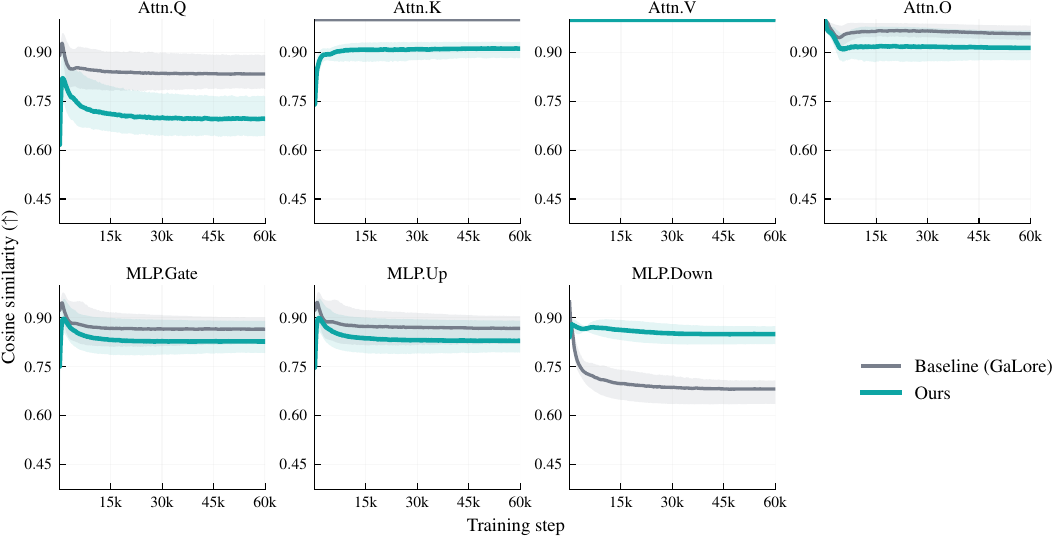}
  \caption{\textbf{Cosine similarity trajectories on Qwen2.5 350M.} The same analysis is applied to Qwen2.5.}
  \label{app_fig:cosine-qwen25}
\end{figure*}

\begin{figure*}
  \centering
  \includegraphics[width=\textwidth]{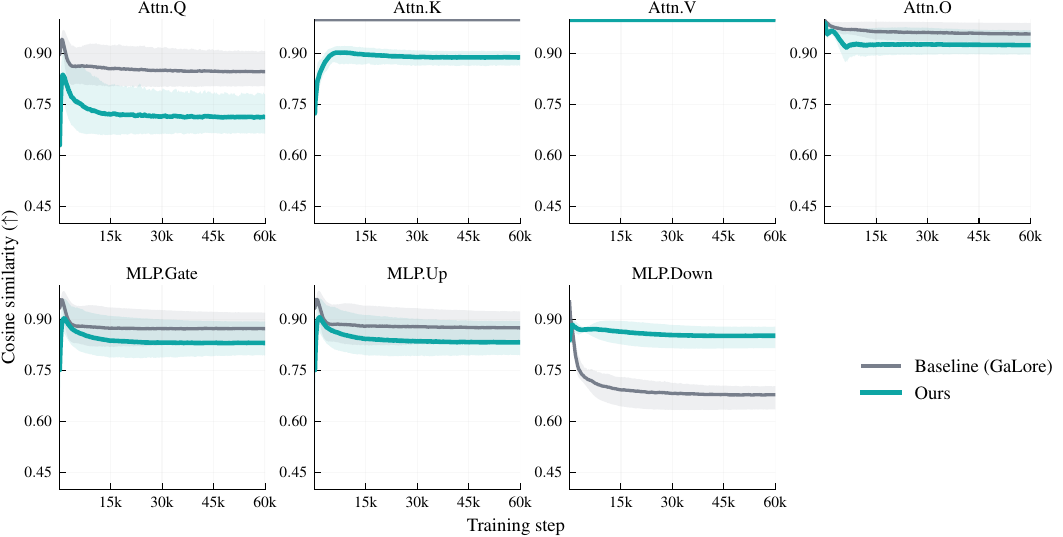}
  \caption{\textbf{Cosine similarity trajectories on Qwen3 350M.} The same analysis is applied to Qwen3.}
  \label{app_fig:cosine-qwen3}
\end{figure*}

\begin{figure*}
  \centering
  \includegraphics[width=\textwidth]{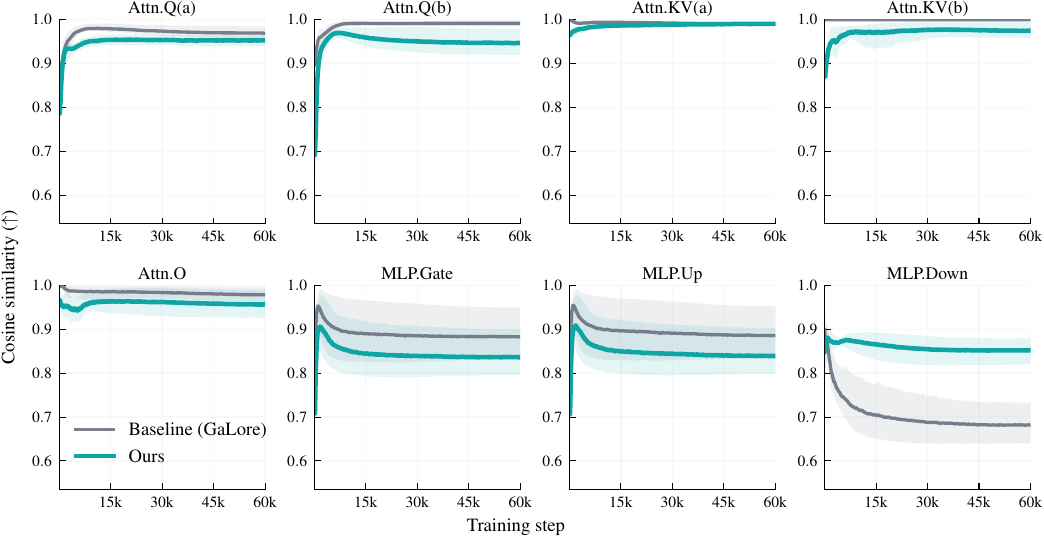}
  \caption{\textbf{Cosine similarity trajectories on DeepSeek-V2 350M.} The same analysis is applied to DeepSeek-V2.}
  \label{app_fig:cosine-deepseekv2}
\end{figure*}

\end{document}